%% file: main.tex
\documentclass[]{article}
\input{header}
\input{macros.sty}
\title{\textbf{Towards an efficient and risk aware strategy for guiding farmers in identifying best crop management}}
\input{authors}
\date{\today}

\begin{document}

\maketitle
\begin{abstract}
\input{corpus/abstract}
\end{abstract}
\section{Introduction}
\label{sec:intro}
\input{corpus/introduction}
\section{Methods}
\label{sec:methodsBatchBandit}
\input{corpus/methods}
\section{Results}
\label{sec:resultsBatchBandit}
\input{corpus/results}
\section{Discussion}
\label{sec:discussionBatchBandit}
\input{corpus/discussion}
\section{Conclusion}
\label{sec:conlusion}
\input{corpus/conclusion}
\section*{Software availability}
\input{corpus/softwareAvailability}
\bibliography{biblio.bib}

\clearpage
{\center \Huge \textbf{Supplementary Materials}}
\begin{appendix}
\counterwithin{figure}{section}
\counterwithin{table}{section}
\counterwithin{algorithm}{section}
\counterwithin{equation}{section}

\section{Maize simulations}
\label{sec:simulationsBatchBandit}
\input{appendices/maizeSimulations}

\section{Algorithms}
\label{sec:algorithmsBatchBandit}
\input{appendices/algorithms}

\section{Experiment complements}
\label{sec:additionalExpBatchBandit}
\input{appendices/additionalExp}

\section{Theoretical Analysis}
\label{sec:demoBatchBandit}
\input{appendices/regretBatch}

\section{Alternative performance measure of fertilizer practices}
\label{sec:YEcomplement}
\input{appendices/YEcomplement}

\end{appendix}
\end{document}

%% file: header.tex
\usepackage[utf8]{inputenc}

\usepackage{graphicx}
\usepackage{natbib}
\usepackage[margin=.95in]{geometry}
\usepackage{xcolor}
\definecolor{ultramarine}{rgb}{0.07, 0.04, 0.56}
\usepackage[colorlinks = true,
            linkcolor = ultramarine,
            urlcolor  = ultramarine,
            citecolor = ultramarine,
            anchorcolor = ultramarine,
            unicode]{hyperref}
\usepackage{siunitx}
\usepackage{booktabs}
\usepackage{multirow}
\usepackage{tabularx}
\usepackage{mdframed}
\usepackage{enumitem}
\usepackage{subcaption}

\usepackage{todonotes}

\usepackage[algo2e]{algorithm2e}

\SetCommentSty{mycommfont}
\makeatletter
\newcommand{\nosemic}{\renewcommand{\@endalgocfline}{\relax}}
\newcommand{\dosemic}{\renewcommand{\@endalgocfline}{\algocf@endline}}
\makeatother

\usepackage[title]{appendix} 

\usepackage{chngcntr}

\usepackage{authblk}

%% file: authors.tex
\author[1,2,3]{Romain Gautron \thanks{\href{mailto:romain.gautron@cirad.fr}{romain.gautron@cirad.fr}}}
\author[4]{Dorian Baudry}
\author[5,6,7]{Myriam Adam}
\author[1,2,8]{Gatien N. Falconnier}
\author[1,2,9]{Marc Corbeels}

\affil[1]{AIDA, Univ Montpellier, France.}
\affil[2]{CIRAD, Montpellier, France.}
\affil[3]{CGIAR Platform for Big Data in Agriculture, Alliance of Bioversity International and CIAT, Km 17, Recta Cali Palmira 763537, Colombia.}
\affil[4]{Univ. Lille, CNRS, Inria, Centrale Lille, UMR 9198-CRIStAL, F-59000 Lille, France.}
\affil[5]{CIRAD, UMR AGAP Institut, Bobo-Dioulasso 01, Burkina Faso.}
\affil[6]{UMR AGAP Institut, Univ Montpellier, CIRAD, INRAE, Institut Agro, Montpellier, France.}
\affil[7]{Institut National de l’Environnement et de Recherches Agricoles (INERA), Burkina Faso.}
\affil[8]{International Maize and Wheat Improvement Centre (CIMMYT)-Zimbabwe, 12.5 km Peg Mazowe Road, Harare, Zimbabwe.}
\affil[9]{International Institute of Tropical Agriculture, PO Box 30772, Nairobi, 00100, Kenya.}

%% file: corpus/abstract.tex
Identification of best performing fertilizer practices among a set of contrasting practices with field trials is challenging as crop losses are costly for farmers. To identify best management practices, an “intuitive strategy” would be to set multi-year field trials with equal proportion of each practice to test. Our objective was to provide an identification strategy using a bandit algorithm that was better at minimizing farmers’ losses occurring during the identification, compared with the ``intuitive strategy''. We used a modification of the Decision Support Systems for Agro-Technological Transfer (DSSAT) crop model to mimic field trial responses, with a case-study in Southern Mali. We compared fertilizer practices using a risk-aware measure, the Conditional Value-at-Risk (CVaR), and a novel agronomic metric, the Yield Excess (YE). YE accounts for both grain yield and agronomic nitrogen use efficiency. The bandit-algorithm performed better than the intuitive strategy: it increased, in most cases, farmers’ protection against worst outcomes. This study is a methodological step which opens up new horizons for risk-aware ensemble identification of the performance of contrasting crop management practices in real conditions.

%% file: corpus/introduction.tex
Identifying site-specific best-performing crop management is crucial for farmers to increase their income from crop production, but also for minimizing the negative environmental impact of cropping activities \citep{tilman2002agricultural}. However, due to weather variability, the identification of these practices can be challenging, in particular with rainfed farming: what worked best in a wet year, might not work in the next season, when rainfall is less \citep[][]{affholder1995effect}. In fact, the performance of crop management at a given site has an underlying “hidden” distribution due to inter-annual weather variability, thus creating great uncertainty \citep[][]{fosu2012simulating}. Because crop management decisions are recurrent, i.e. they are repeated for each new crop growing season, the identification of optimal crop management falls into the category of sequential decision making under uncertainty \citep[][]{gautron2022reinforcement}. Computer-based decision support tools  can allow farmers to make more informed (less uncertain) decisions about their cropping practices from one year to the next, and can facilitate farmers’ risk management in the face of seasonal weather variability  \citep{hochman2011emerging}. There exist numerous decision support tools of widely ranging complexity for crop management, introduced to farmers with varying degrees of success \citep[][]{gautron2022reinforcement}.

Machine learning (ML) and more generally artificial intelligence (AI) can help address sequential decision making under uncertainty. In particular, the bandit algorithm paradigm \citep{lattimore2020bandit} considers a decision-maker, called agent, who repeatedly faces a choice between contending actions, and has to iteratively improve its decisions with trials. The canonical bandit problem originates from clinical trials with sequential drug allocation \citep{thompson1933likelihood}. At each time step, the agent chooses one action (i.e.\@, one drug for a patient) amongst a set of possible actions. Each action provides a reward (i.e.\@; tumor cell reduction after taking the drug), drawn from a corresponding unknown reward distribution (i.e.\@, the distribution of tumor cell reduction for the drug). The optimal action has the reward distribution with the highest mean reward (i.e.\@, the highest mean tumor cell reduction). The objective of the agent is to sequentially choose actions such that the expected sum of rewards is maximized. Maximizing the total expected rewards is equivalent to minimizing the regret, which is a measure of the total losses that occur with sub-optimal actions \citep{robbins1952some}. 

Iteratively, the agent refines his next decision based on all previous results. To know how a given action performs, a sufficient number of (possibly poor) rewards is required: this is the exploration phase. To maximize the expected sum of rewards, the previous actions that provided good results so far must be selected more frequently; this is the exploitation phase. Bandit algorithms aim at finding the right balance between exploration and exploitation. This  \textit{exploration-exploitation dilemma} is a reality for farmers when implementing crop management. Farmers typically want to minimize overall crop yield losses and typically explore the performance of promising new crop management practices on small test plots \citep{cerf2006outils, evans2017data}. They avoid potentially large crop yield losses from new management by managing a gradual transition between the current management and the promising new one(s), based on the results they obtain on the small test plots.

The objective of this paper is to develop a novel strategy to identify best crop management. We set as baseline an “intuitive strategy” which consists in identifying the best crop management through multi-year field trials in which a set of crop management practices is tested in an equiproportional way. We compare this “intuitive strategy” to a novel crop management identification strategy, based on a bandit algorithm. This novel identification strategy aims to minimize farmers’ yield losses occurring during the identification process, compared to the intuitive strategy. Thus, we test the hypothesis that bandit algorithm can help farmers to better identify the best crop management for their context, while further minimizing crop yield losses related to sub-optimal choices in new crop management.

Our case study considers the rainfed maize production in southern Mali, and we compare the performance of both crop management identification strategies based on maize growth simulations using a calibrated crop model in order to mimic real-world performance of crop management. The novel identification strategy does, however, not depend on model simulations, and ultimately aims at being applied in real field conditions. As for crop management, we focus on nitrogen fertilization. Tailoring nitrogen fertilizer recommendations to farmers’ contexts is known to be challenging. Indigenous soil nitrogen supply, depending to a large extent on past-season events, is not accurately known to farmers, whilst in-season nitrogen mineralization depends largely on weather events\citep{morris2018strengths}, themselves uncertain. Crop nitrogen requirements, such as with maize, are related to specific crop growth stages \citep{hanway1963growth} and excessive mineral nitrogen supply can induce nitrate leaching, especially in wet conditions \citep{meisinger2002principles}. Therefore, there are \textit{a priori} no upfront optimal nitrogen fertilizer practices.

%% file: corpus/methods.tex
\subsection{Virtual crop management identification problem}
\label{sec:learningProcess}
In our virtual crop management identification problem, a population or ensemble of farmers joined a participatory experiment to identify the best nitrogen fertilizer practices for maize production in their region, Koutiala in southern Mali. A total population of 500 farmers was considered. The distribution of soil types of the fields associated with the group of farmers was representative of the region (Table~\ref{tab:soils}). A total population of 500 farmers was considered. Each farmer belonged to a cohort that corresponded to an ensemble of farmers growing maize on the same soil type. For each cohort, we wanted to identify the best nitrogen fertilizer practice from a set of recommended practices (see Table~\ref{tab:optimalpractices} and Section~\ref{sec:YEandCVaR} for the performance metrics we considered). The research team set the additional objective to limit the crop yield losses of individual farmers that could arise from poor nitrogen fertilizer practice recommendations during the identification process.

At the beginning of each crop growing season, we assumed that a random number of farmers (uniformly obtained between 250 and 350) of the population of 500 farmers volunteered to apply the recommended fertilizer applications provided by the research team. Each year, the group of volunteers was variable in size and in the representation of cohorts, as could occur in reality (Figure~\ref{fig:samplingGroup}). Thus, researchers did not control the composition of the group of volunteers. Each farmer indicated the fields and corresponding soils on which she/he planned to grow maize. Researchers then provided a fertilizer recommendation (Table~\ref{tab:optimalpractices}) to each farmer for the ongoing season, depending on her/his soil i.e.\@ cohort. At the end of the season, volunteer farmers shared their results in terms of maize grain yields with the research team, allowing to refine the recommendations for the next season. The whole process was repeated during 20 consecutive years following the same process (Figure~\ref{fig:decisionProblem}).

\begin{figure}
    \centering
    \includegraphics[width=.5\textwidth]{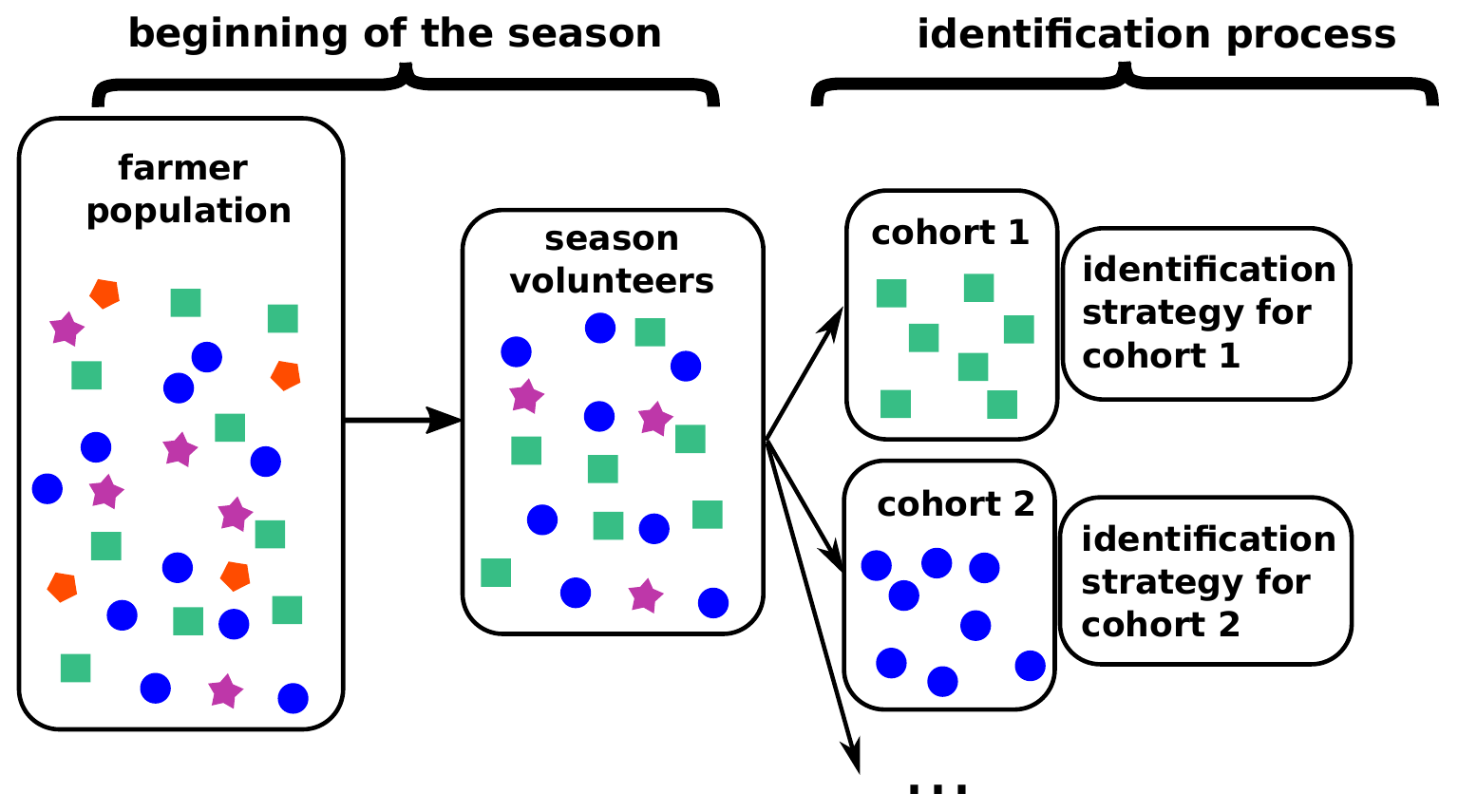}
    \caption{Yearly process to generate nitrogen fertilizer recommendations: at the beginning of the crop ping season. Individuals from the overall farmer population volunteered to test a fertilizer practice. Similar symbols represent a cohort, i.e.\@, a group of farmers having fields with the same soil type. The group of volunteer farmers was broken down by cohort and researchers independently generated fertilizer recommendations for each cohort. Researchers did not control the number of volunteers from the respective cohorts In this example, only three of the four possible cohorts are found in the volunteer group.}
    \label{fig:samplingGroup}
\end{figure}%

\begin{figure}
    \centering%
    \begin{subfigure}[t]{.48\textwidth}
        \centering
        \includegraphics[width=\textwidth]{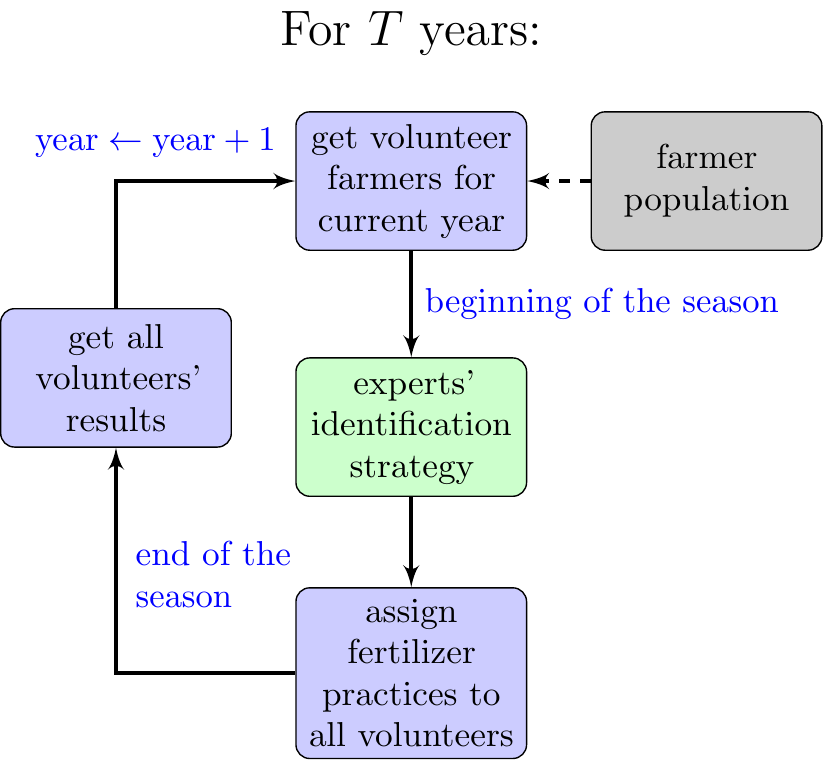}
        \caption{Diagram of the ensemble best fertilizer identification process. Each year, a group of volunteer farmers test fertilizer practices recommended by experts and contribute to identifying the best fertilizer practices for the region. At the end of each season, the farmers share their results with experts. The experts will use these results to improve their recommendations for the next growing season. The process repeats for a total number of $T$ years.}
        \label{fig:decisionProblem}
    \end{subfigure}
    \hfill%
    \begin{subfigure}[t]{.48\textwidth}
        \includegraphics[width=.7\textwidth]{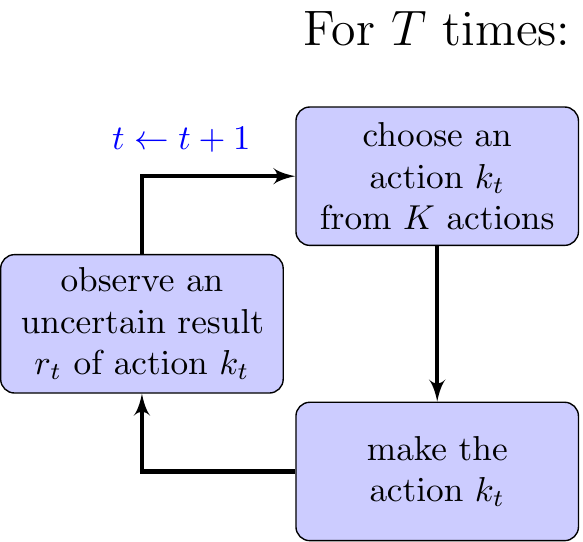}
        \caption{Canonical bandit problem. For $T$ times, an agent sequentially makes a decision on an action $k_t$ from the set $\{1, \cdots, K\}$ of possible actions. After making the action $k_t$, the agent observes an uncertain result $r_t$. This result is sampled from a fixed distribution, unknown to the agent, which corresponds to the effect of action $k_t$.
        }
        \label{fig:banditProblem}
    \end{subfigure}
\caption{Schematic representation of the ensemble best fertilization identification process and the canonical bandit problem.}
\label{fig:banditProblems}
\end{figure}

\paragraph{Nitrogen fertilizer practices.}Ten nitrogen fertilizer practices were considered as recommendations in the virtual modeling experiment (see Table~\ref{tab:practices}). Practices 0 to 7 explored the following set of split applications for a total amount of 135 kg N/ha applied:
\begin{itemize}[label=-]
    \item Two split applications (practice 0):  15 days after planting (DAP) and 30 DAP.
    \item Three split applications (practice 4) :15 DAP, 30 DAP and 45 DAP.
    \item Split applications according to the rainfall amount (practices 2, 3 and 6, 7): 2nd and 3rd top-dressing applications only if the cumulated rainfall amount from the start of the season to 30 DAP exceeds the 30th percentile of historical rainfall i.e. 200 mm.
    \item Split applications according to plant nitrogen content (practices 1, 3 and 5, 7): 2nd and 3rd top-dressing applications only if the simulated nitrogen stress factor (\texttt{NSTRES} in DSSAT, see below) exceeds 0.2 (0 no stress, 1 maximal stress), hereby mimicking the use of a portable chlorophyll meter to monitor plant nitrogen content \citep[e.g.\@][]{kalaji2017comparison}.
\end{itemize}
Practice 8  corresponded to the optimal fertilization for maize  (70 kg N/ha) in the study area based on simulations  \citep{huet2022coping} , i.e.\@ the average of the N fertilizer rates that were observed to result in maximum positive return on fertilizer investment \citep{getnet2016yield}. Finally, practice 9 (180 kg N/ha) corresponded to a nitrogen fertilizer practice that is likely excessive. For all these practices, the nitrogen fertilizer applied was assumed to be ammonium nitrate broadcasted on the soil surface.

\input{tables/soils}

\input{tables/practices}

\paragraph{Maize growth simulations.}In order to get a proxy for real-world performances of the maize nitrogen fertilizer practices, we simulated maize growth responses to fertilization under the growing conditions of Koutiala in southern Mali using \texttt{gym-DSSAT} \citep{gymdssat}.  \texttt{gym-DSSAT} is a modification of the DSSAT crop simulator \citep{hoogenboom2019dssat} to allow a user to read DSSAT internal states and take daily fertilization decisions during the simulations (e.g.\@ based on DSSAT internal states). For each soil type in Table~\ref{tab:soils} that was parametrized in DSSAT using the data from \citet{adam2020more}, each simulated maize grain yield value is a sample of the response distribution for the considered fertilizer practice. This response distribution is the result of weather variability, generated in our study by the stochastic weather generator WGEN \citep{richardson1984wgen, soltani2003statistical}, which was calibrated using the 47-year-long weather records from N’tarla, about 30 km from Koutiala \citep{ripoche2015cotton}. The `sotubaka' maize cultivar (from the DSSAT default cultivar list) was used for all model simulations as a representative of maize variety in southern Mali. Water and nitrogen stresses were simulated, but yield reduction through pests and diseases were not considered, neither was weed competition. In the model simulations, a different weather time series was generated for each growing season and for each recommendation using WGEN, inducing independent simulated maize yield responses to nitrogen fertilization. Section~\ref{sec:simulationsBatchBandit} of Supplementary Materials gives further details of the simulation settings.

We simulated $10^5$ times the maize grain yield responses to a given fertilizer practice for the differet soil types, which corresponds to $10^5$ hypothetical growing seasons. These samples were used i) to ensure that simulated maize yield responses were in realistic expected ranges, ii) to qualitatively evaluate the complexity of the decision problem, and iii) to determine best nitrogen fertilizer practices  whilst analyzing the performance of the crop management identification strategies. The samples were not provided to the algorithms prior to their application (i.e.\@ no prior knowledge of the problem).

\subsubsection{Performance indicators of fertilizer practices}
\label{sec:YEandCVaR}
A criterion to evaluate both the economic and environmental performance of a fertilizer practice $\pi$ is Agronomic Nitrogen use Efficiency (ANE), as defined in \citet{vanlauwe2011agronomic}:
\begin{equation} \label{eq:ANE}
\text{ANE}^{\pi} :=\frac{\text{Y}^{\pi} - \text{Y}^0}{\text{N}^{\pi}}
\end{equation}
where $\text{Y}^{\pi}$ is the crop yield obtained with the nitrogen fertilizer practice $\pi$ which required a quantity $\text{N}^{\pi}$ of nitrogen and $\text{Y}^{0}$ is the yield of the control obtained in the same conditions without nitrogen fertilization. Maximising ANE is a proxy of minimizing the quantity of nitrogen losses, e.g.\@ through nitrate leaching.

However, ANE has some limitations: for example, an ANE value of 25 kg grain/kg N can be achieved with a fertilizer input of 20 kg N/ha yielding a total yield gain of 500 kg/ha, or with an input of 60 kg N/ha yielding a total gain of 1500 kg/ha. For the same ANE, a farmer is likely to prefer the fertilizer practice that provides the greatest crop yield gain, i.e.\@ with 60 kg N/ha. Similarly, choosing fertilizer practices only based on the associated crop yield gains is not satisfying. A similar yield gain can be achieved with different nitrogen fertilizer input rates which result in fairly different ANE: the practice with the highest efficiency must be preferred as it required less nitrogen fertilizer to achieve the same yield gain.

We built the Yield Excess (YE) indicator that favors the nitrogen fertilizer practice with the highest yield gain for those practices sharing the same ANE, and favors the practice with the highest efficiency for those practices sharing the same yield gain. YE of a nitrogen fertilizer practice $\pi$ with respect to the reference practice $\pi_{\text{ref}}$ of constant efficiency $\text{ANE}_{\text{ref}}$ using the same  quantity of nitrogen fertilizer as practice $\pi$, denoted $\text{N}^{\pi}$, is computed as follows:
\begin{align}
        \text{YE}^{\pi}  :=& \text{~Y}^{\pi} - \text{~Y}^{\pi_{\text{ref}}} \label{eq:YE0}\\
        =& \underbrace{\text{~Y}^{\pi} - \text{~Y}^{0}}_{\substack{\text{yield gain of $\pi$} \\ \text{w.r.t. control}}} - \underbrace{\big(\text{~Y}^{\pi_{\text{ref}}} - \text{~Y}^{0}\big)}_{\substack{\text{yield gain of $\pi_{\text{ref}}$} \\ \text{w.r.t. control}}}\\
        =& \text{~Y}^{\pi} - \text{Y}^0 - \text{N}^{\pi} \times \text{ANE}_{\text{ref}} \label{eq:YE1}\\
        =& \big(\text{Y}^{\pi} - \text{Y}^0\big) \times \underbrace{\Big(1 - \frac{\text{ANE}_{\text{ref}}}{\text{ANE}^{\pi}}\Big)}_{\substack{\text{penalization factor}}} \label{eq:YE2}
\end{align}

\paragraph{}The YE of practice $\pi$ with respect to the reference practice $\pi_{\text{ref}}$ corresponds to the yield difference between the practice $\pi$ and a reference practice that has a constant ANE equal to  $\text{ANE}_{\text{ref}}$ and which uses the same quantity $\text{N}^{\pi}$ of nitrogen fertilizer as $\pi$.  $\text{YE}^{\pi}$ increases with $\text{ANE}^{\pi}$ (Figure~\ref{fig:YEsurface}). $\text{YE}^{\pi}$ is negative and decreases with $\text{Y}^{\pi} - \text{Y}^0$ when $\text{ANE}^{\pi} < \text{ANE}_{\text{ref}}$ and is positive and  increases with $\text{Y}^{\pi} - \text{Y}^0$ when $\text{ANE}^{\pi} \geq \text{ANE}_{\text{ref}}$. The YE of fertilizer practices with efficiency below $\text{ANE}_{\text{ref}}$ are negatively affected by this metric. 
We chose $\text{ANE}_{\text{ref}}=15$ kg grain/kg N for our model simulation experiments, the average ANE currently achieved by farmers across sub-Saharan Africa \citep{ten2019maize, vanlauwe2011agronomic}.

\begin{figure}
    \includegraphics[width=\textwidth]{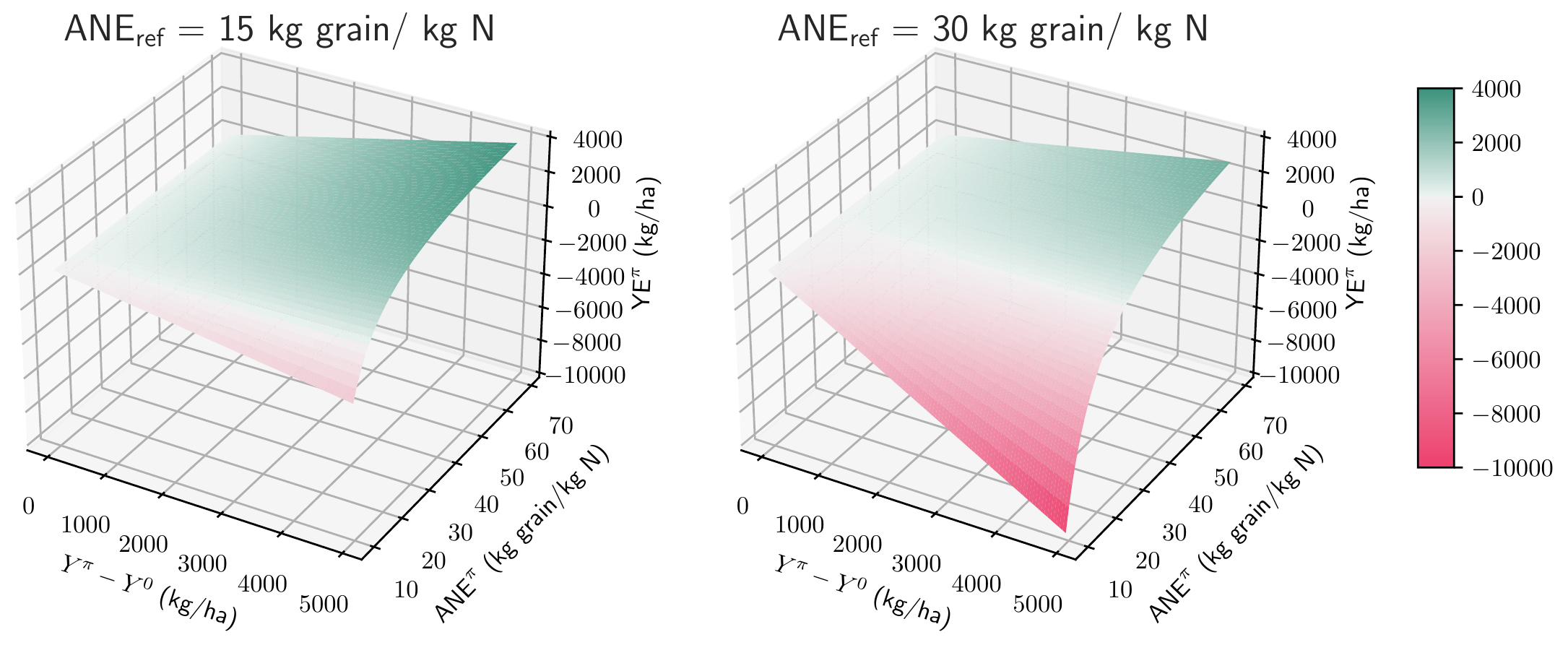}
    \caption{Yield Excess ($\text{YE}^{\pi}$,  Equation~\ref{eq:YE2}) for $\text{ANE}_{\text{ref}}=15$ kg grain /kg N and $\text{ANE}_{\text{ref}}=30$ kg grain /kg N. $\text{Y}^{\pi}$ is the maize grain yield obtained with nitrogen fertilizer practice $\pi$, $\text{Y}^{0}$ is the yield obtained with no nitrogen fertilization (control). $\text{ANE}^{\pi}$ is the Agronomic Nitrogen use Efficiency of the nitrogen fertilizer practice ${\pi}$ (Equation~\ref{eq:ANE}).}
    \label{fig:YEsurface}
\end{figure}

Because farmers are usually risk averse \citep[e.g.\@][]{cerf1997approche, menapace2013risk, jourdain2020farmers}, they are likely to prefer, for example, a stable maize grain yield of 3000 kg/ha rather than a yield of 5000 kg/ha in half of the years, and of 1000 kg/ha in the other half of the years, while both distributions share the same expectation.
To account for risk aversion, we computed the Conditional-Value-at-Risk \citep[CVaR,][]{mandelbrot1997variation, acerbi2001cvar}, a risk-aware metric that originated from finance. The CVaR focuses on the lower tail of the distribution\footnote{Two definitions of the CVaR coexist in the literature, depending if an outcome is considered as a gain or a cost \citep{dowd2007measuring}. We adopted the gain point of view.}. For a (continuous) random variable X with cumulative distribution function $F_X$, we call Value-at-Risk (VaR) of level $\alpha$ the quantile of probability $\alpha \in (0,1]$ of X, defined as:
\begin{equation}
\text{VaR}_{\alpha}(X) := \inf{\{x \in \mathbb{R} : F_X(x) > \alpha\}}
\end{equation}
Then the CVaR of X of level $\alpha \in (0,1]$ is the mean value of the left tail of X of probability $\alpha$, defined as:
\begin{equation}
    \cvar{\alpha}(X) := \Esp[ X \vert X \leq \text{VaR}_{\alpha}(X)]\,  \label{eq:cvar}
\end{equation}

A decision maker would choose the option with the highest CVaR for the considered level $\alpha$. The more $\alpha \to 0^+$, the more the metric focuses on the worst observable yields. On the contrary, the more $\alpha \to 1$, the less risk averse is the measure. When $\alpha = 1$, the CVaR equals the usual expectation $\Esp{[X]}$, which is risk neutral (Figure~\ref{fig:cvar}). In our model simulation experiments, we chose $\alpha=30\%$. The $\cvar{30\%}$ represents the mean crop yield of the 30\% worst observable years.

\newcommand\scaleboxpara{.7}
\begin{figure}
	\begin{subfigure}[t]{0.5\textwidth}
		\centering
		\includegraphics[width=.8\textwidth]{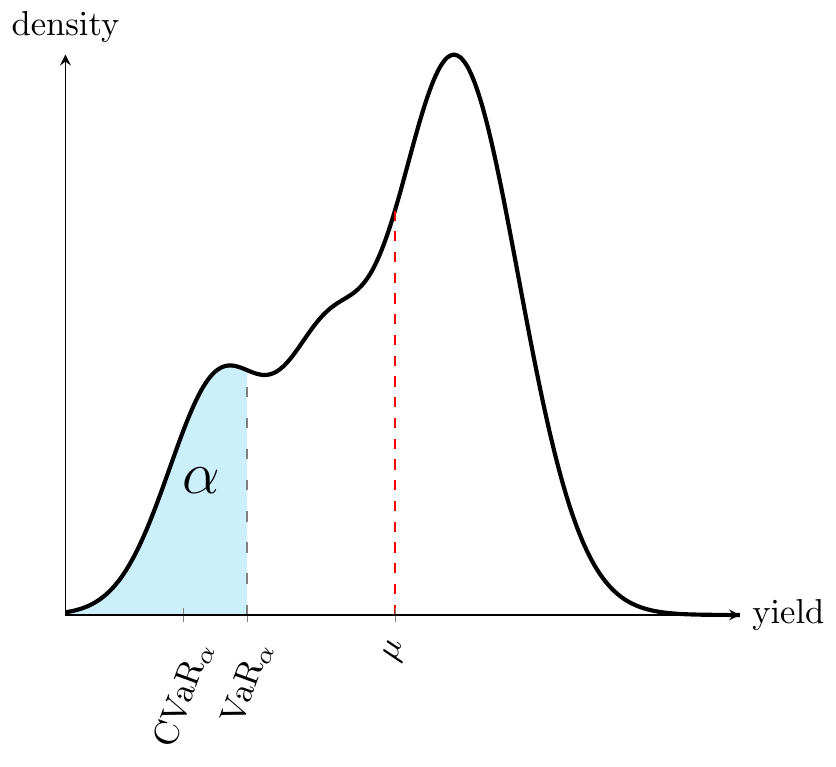}
		\caption{High risk aversion ($\alpha \approx 20 \%$)}
		\label{fig:cvarHigh}
	\end{subfigure}
		\begin{subfigure}[t]{0.5\textwidth}
		\centering
		\includegraphics[width=.8\textwidth]{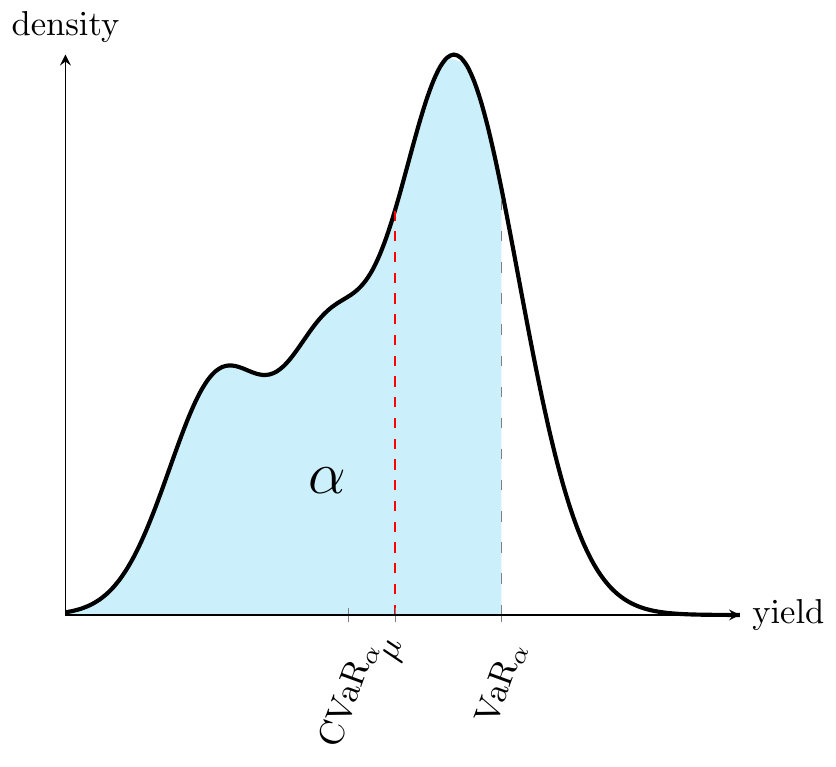}
		\caption{Low risk aversion ($\alpha \approx 80 \%$)}
		\label{fig:cvarLow}
	\end{subfigure}%
    \caption{The Conditional Value-at-Risk (CVaR) of level  $\alpha$ is the mean value of the blue area of the distribution of probability $0 < \alpha \leq 1$ . $\text{VaR}_{\alpha}$ stands for Value-at-Risk of level $\alpha$ and is the quantile of probability $\alpha$ of the distribution. The more $\alpha \to 1$, the more risk neutral is the CVaR. $\mu$ represents the mean value of the distribution which equivalent to the CVaR of level $\alpha=100\%$.}
\label{fig:cvar}
\end{figure}

\subsection{Identification of the best fertilizer practices}
\label{sec:learningBestFertilization}
\paragraph{The canonical and batch bandit problems}The ensemble identification of the best crop management practices with the constraint of minimizing farmers' crop yield losses occurring during the identification process (Section~\ref{sec:learningProcess}) can be modelled as a special type of bandit problems. The canonical bandit problem, which is the cumulated regret minimization (see Introduction), assumes that at each time step, a single trial is made and is followed by a single observation of a result, in a purely sequential mode. In contrast, the batch bandit setting \citep{perchet2015batch} assumes that at each time step an ensemble of trials are conducted in parallel, followed by the observation of an ensemble of results. Figure~\ref{fig:banditProblems} illustrates on the one hand the ensemble identification process of best crop fertilizer practices (Figure~\ref{fig:decisionProblem}), modelled as a batch-bandit problem, and the on other hand the canonical bandit problem  (Figure~\ref{fig:banditProblem}). 

In the canonical bandit problem, the agent goal is to maximize the expectation of the sum of rewards that were collected since the first decision. The agent objective can be formalized as maximizing $\Esp\Big[\sum\limits_{t=1}^T r_t\Big]$ for any time horizon $T\geq1$, with $r_t$ the reward the agent has collected at time $t$. On the other hand, bandits that are \textit{risk-aware}  \citep{cassel2018general}, the agent maximizes a risk-aware measure of the collected rewards, such as the CVaR (Section~\ref{sec:YEandCVaR}), instead of the expectation of rewards. Our ensemble fertilizer decision problem can be described as a \textit{risk-aware batch-bandit} decision problem.

\paragraph{The ensemble identification problem of best fertilizer practices}In our virtual modeling experiment, for $t \in \{1, 2, \cdots, T\}$, at each season $t$, researchers assigned each $n_t$ volunteer farmers for season $t$ with a nitrogen fertilizer practice $\pi \in \{1, 2, \cdots, K\}$. Each farmer belonged to a cohort $c \in \{1, 2, \cdots, C\}$. At the end of season $t$, researchers assemble rewards $Y_t=\{y^1_t,\dots, y^{n_t}_t\}$ as a result of the fertilizer practices of all farmers for season $t$. For each cohort $c \in \{1, \cdots, C\}$, rewards are independently and identically distributed from unknown stationary distributions $\{ \nu_1^c, \cdots, \nu_K^c \}$. These reward distributions are the YE with $\text{ANE}_{\text{ref}}=15$ kg grain/kg N associated to each of the ten recommended nitrogen fertilizer practices, for a given soil type. We denote $\cY_T= \bigcup_{t=1}^{T} Y_t$ the set of all rewards observed by all farmers between $t=1$ and $t=T$. The objective of an identification strategy is to maximize, for a given CVaR level $\alpha$ and any time horizon $T\geq1$:
\begin{equation}
    \Esp[\text{CVaR}_{\alpha}(\cY_T)]
\label{eq:farmerObj}
\end{equation}
For each cohort $c \in \{1, \cdots, C\}$, an optimal nitrogen fertilizer practice $\pi^c_*$ is given by:
\begin{equation}
    \pi^c_* = \argmax{k}~\text{CVaR}_{\alpha}(\nu^c_k)
\end{equation}
Consequently, an optimal identification strategy always assigns nitrogen fertilizer practice $\pi^c_*$ to all farmers belonging to cohort $c$.

\subsubsection{Identification strategies}\label{sec:samplingStrategies}
We expected fertilizer practices to perform differently within each cohort, i.e.\@ each soil. For example, the optimal nitrogen fertilizer practices were expected to be different between a cohort growing maize on a shallow sandy soil and a cohort growing maize on a deep clayey soil. Consequently, the results of one cohort were not supposed to be directly relevant for another cohort. Each soil was considered as an independent identification problem, i.e.\@ had its own independent identification strategy which did not share information with the identification strategies of other soils.

For a given soil, from one season to another, the identification strategy kept memory of all results observed during past seasons, for the same soil. In model simulation experiments, we considered two types of identification strategies: either the standard \texttt{ETC} (Explore-Then-Commit) strategy, previously referred as the ``intuitive strategy'', or \texttt{BCB}, the bandit-algorithm based identification strategy. For the seven soils in Table~\ref{tab:soils}, the identification strategy types were either all \texttt{ETC}, or all \texttt{BCB}, but not a mix of both.

\paragraph{Intuitive identification strategy}\texttt{ETC} provides a simple and intuitive solution to the exploration-exploitation dilemma. During an initial exploration phase of an arbitrary number of years, \texttt{ETC} equiproportionally test all nitrogen fertilizer strategies. Thereafter, the exploitation phase starts and \texttt{ETC} chooses for the remaining time the fertilizer strategy that has shown best performance during the exploration phase. In Section~\ref{sec:ETC} of Supplementary Materials, we provide a simple adaptation of \texttt{ETC} to the batch setting (see Section~\ref{sec:learningProcess}) using the CVaR of rewards rather than the classical expectation. We considered \texttt{ETC-3} and \texttt{ETC-5}, with respectively 3 and 5 years of exploration phases. During the exploration phase, fertilizer practices are randomly assigned in equal proportions to the farmers within the cohort.

\paragraph{Bandit based identification strategy}\texttt{BCB} is a risk-aware bandit algorithm \citep{cassel2018general} which uses the CVaR of rewards as decision criterion, in the batch bandit setting. \texttt{BCB} derives from the the work of \citet{baudry2021optimal}. We provide the pseudo-code of \texttt{BCB} and detail how it works in Supplementary Materials Section~\ref{sec:bcvtsDetailed}. The general idea of the bandit algorithm is, for each season, to leverage the information acquired during all past seasons, such that the algorithm adapts to optimally manage the exploration-exploitation dilemma.


We provide a quick overview of the execution of \texttt{BCB} with algorithm~\ref{algo:bcbPseudoCode}. Considering the YE with $\text{ANE}_{\text{ref}}=15$ kg grain/kg N as results, we set its maximum observable result to 4000 kg/ha for all fertilizer practices as required for the execution of \texttt{BCB} (see first execution step of algorithm~\ref{algo:bcbPseudoCode}), based on Figure~\ref{fig:YEsurface}. As an additional feature, \texttt{BCB} provides a fair distribution of risky option trials amongst farmers at the cohort level. The bandit algorithm ranks each fertilizer practice according to its observed performance in the previous year. The algorithm then recommends first the practices that appear to yield best results to the farmers that have experienced worst results so far.

\input{algorithms/bcb}

\subsubsection{Direct measure of performance of an identification strategy}\label{sec:empiricalMeasure}
We denote $\widehat{C}_{\alpha}$ the expression of the empirical CVaR of level $\alpha \in (0, 1]$. The empirical CVaR is an estimate of the true CVaR as defined in Equation \ref{eq:cvar} --just as an average value is an estimate of the true mean of a distribution--. Assuming a sample $\cY$ of rewards sorted in an increasing order i.e.\@ $\cY = \{y_1, \cdots, y_n\} \text{~such that~} y_i \leq y_{i+1}$, and defining $q = y_{\lceil\alpha n\rceil}$ the empirical quantile of level $\alpha$, we have:
\begin{equation}
\label{eq:cvaremp}
\widehat{C}_{\alpha}(\cY) := q - \frac{1}{n \alpha}\sum_{i=1}^{n}\max(q - y_i, 0)
\end{equation}
In a simulated problem, the quantity in Equation~\ref{eq:farmerObj} can be estimated by repeatedly applying $R$ times an identification strategy during $T$ years, and then concatenating all results of all farmers from time $t=1$ to time $t=T$ for all replications, and finally computing the empirical CVaR of the resulting set. In order to approximate all expectations, for all experiments, in practice we consider $R=960$ (12 executions in parallel on an 80 core machine; for each one of the 960 experiments, the weather generator had a different random state). We denote $r \in \{1, \cdots, R\}$ the repetition index. We define $\mathfrak{Y}_T=\bigcup_{r=1}^{R} \mathcal{Y}_T^r$ i.e.\@ the results of all farmers until year $T$ for all replications. Then:
\begin{equation}
\label{eq:farmerObjEst}
    \Esp{[\text{CVaR}_{\alpha}(\cY_T)]} \mathrel{\widehat{=}} \widehat{C}_{\alpha}(\mathfrak{Y}_T)
\end{equation}
The resulting quantity is an average measure of the results of the group. The more an identification strategy maximizes this quantity, the better it is. In a real-world problem, only one realization of $\text{CVaR}_{\alpha}(\cY_T)$ is computable.

\subsubsection{Proxy measure of performance of identification strategy}\label{sec:syntheticMeasure}
While the quantity in Equation \ref{eq:farmerObj} can be estimated with Equation \ref{eq:farmerObjEst}, it is intricate to analyze and derive statistical guarantees for this estimator. This is why, in the following, we introduce a proxy of this quantity called the cumulated CVaR regret, which is a central element behind the theoretical performance guarantees of bandit algorithms.
The cumulated regret is also a convenient statistic to represent the performance of an algorithm, with little noise.

\paragraph{Mean cumulated regret of the farmer population}Considering a single cohort $c$, we suppose that we sequentially repeat $T$ times the choice of one option $k$ from an ensemble of $K$ possible options. Here $k$ is the index of the fertilizer practice. We denote $\cvar{\alpha}^{}(\nu^c_k)$ the CVaR of level $\alpha$ associated with the option $k$ and cohort $c$, and ${\cvar{\alpha}^{}(\nu^c_*)=\max\limits_{k \in \{1, \cdots, K\}}\cvar{\alpha}^{}(\nu^c_k)}$ the highest CVaR at level $\alpha$ of all options for cohort $c$ i.e.\@ the CVaR of the best option for cohort $c$. In expectation, for a farmer belonging to cohort $c$ and following $T$ years the recommendations of a given identification strategy selecting a fertilizer practice $k(t)$ each year $t \in \{1, \cdots, T\}$, we define the cumulated regret for the CVaR as in \citet{tamkindistributionally}:

\begin{align}
\underbrace{R^c_{\alpha}(T)}_{\substack{\text{loss of the} \\ \text{strategy}}} &:= \underbrace{T~\times~\cvar{\alpha}^{}(\nu^c_*)}_{\substack{ \text{score of the best} \\ \text{possible strategy}}} - \underbrace{\Esp{\Bigg[\sum_{t=1}^T \cvar{\alpha}^{}(\nu^c_{k(t)})\Bigg]}}_{\substack{\text{score of the actual} \\ \text{strategy}}} \label{eq:regret1} \\
& = \sum_{k=1}^{K} \underbrace{\big(\cvar{\alpha}^{}(\nu^c_*)-\cvar{\alpha}^{}(\nu^c_{k})\big)}_{\substack{\text{loss between the best option} \\ \text{and the option $k$ for cohort $c$}}}~\times~\underbrace{\Esp{[N_k^c(T)]}}_{\substack{\text{expected number of times} \\ \text{option $k$ is chosen for cohort $c$} \\ \text{during the $T$ years}}} \label{eq:regret2}
\end{align}

For cohort $c$, the cumulated regret $R^c_{\alpha}(T)$ can be seen as a loss occurred with the considered strategy with respect to the best possible strategy --the one that always chooses the fertilizer practice with the best CVaR--. Equivalently, it can be interpreted as a measure of the expected total error due to sub-optimal actions made during a series of $T$ decisions: the more the best option is chosen within the $T$ decisions, the smaller the cumulated regret is. The mean cumulated regret of the total farmer population is given by the cumulated regret of each cohort, weighted by the probability of an individual to belong to this cohort:
\begin{equation}
    R_{\alpha}(T) = \sum_{c=1}^C R_{\alpha}^c(T) \times \Pr(c),~\text{with}~ \sum_{c=1}^C \Pr(c) = 1 \label{eq:regretAllCohorts}
\end{equation}
When extensively testing an identification strategy on a simulated problem, the CVaR of the different options can be approximated with a large enough number of samples or analytically computed, irrespective of the identification strategy. For each cohort, this corresponds to the left-hand side of Equation \ref{eq:regret2}, and is thus supposed to be known. Note that, for a real-world problem, these quantities are unknown --else the decision problem would have been solved--. On the right hand side of Equation \ref{eq:regret2}, the quantity $\Esp{[N_k^c(T)]}$ can be empirically approximated by repeatedly performing experiments with the identification strategy, and averaging the number of times each fertilizer practice has been chosen since time step $T$ for each cohort. Finally, in Equation~\ref{eq:regretAllCohorts}, the proportion of each soil, i.e.\@ cohort, can be found in Table~\ref{tab:soils}. Minimizing the cumulated regret maximizes the quantity in Equation \ref{eq:farmerObj}, as shown by \citet{cassel2018general}. For a given identification strategy, the smaller and less variable the mean cumulated regret of population (Equation~\ref{eq:regretAllCohorts}), the more farmers are guaranteed to maximize their CVaR of YE.

\paragraph{Distribution of the cumulated regret of individual farmers}
The mean cumulated regret of the population given in Equation~\ref{eq:regretAllCohorts} does not indicate the distribution of individual farmer regrets. For each farmer $f$ belonging to cohort $c$, the individual regret after $T$ years for the CVaR of level $\alpha \in (0,1]$ is computed as:

\begin{equation}
\label{eq:individualLoss}
\tilde{R}_{\alpha}^{f,c}(T) := \sum_{k=1}^{K} \underbrace{\big(\cvar{\alpha}^{}(\nu^c_*)-\cvar{\alpha}^{}(\nu^c_{k(t)})\big)}_{\substack{\text{loss between the best option} \\ \text{and the option  $k$}}}~\times~\underbrace{N_{k}^{f,c}(T)}_{\substack{\text{number of times option k} \\ \text{is chosen during T years} \\ \text{for farmer $f$}}} 
\end{equation}
For each cohort $c$, the distribution of $\tilde{R}_{\alpha}^{f,c}(T)$ indicates how the potential  losses due to bad recommendations are distributed amongst farmers.

%% file: tables/soils.tex
\begin{table}
\caption{Main properties of the soil types of the fields of farmers growing maize in Koutiala, Mali \citep{adam2020more}. ‘\texttt{SLOC}.’ stands for soil organic matter  (g C/ 100 g soil, mean value for the  0-30 cm topsoil); `\texttt{SLDR}' stands for soil drainage rate (fraction/day); `\texttt{SLDP}' stands for soil depth (cm); ‘Prop’ stands for the percentage of each soil type present in the study area. \label{tab:soils}}
\centering
\begin{tabular}{p{2cm}p{2.2cm}p{.7cm}p{.7cm}p{.7cm}p{.7cm}}
	\textbf{Soil name}  & \textbf{Texture}& \textbf{\texttt{SLDR}} & \textbf{\texttt{SLOC}} & \textbf{\texttt{SLDP}} & \textbf{Prop.}\\ \midrule
	\texttt{ITML840101} & clay loam       & 0.60 & 0.20                                  & 110                 & 7\%\\
	\texttt{ITML840102} & loam            & 0.60 & 0.45                                  & 100                 & 9\%\\  
	\texttt{ITML840103} & silty loam      & 0.60 & 0.27                                  & 160                 & 21\%\\  
	\texttt{ITML840104} & silty clay loam & 0.25 & 0.70                                  & 105                 & 4\%\\  
	\texttt{ITML840105} & silty clay loam & 0.40 & 0.35                                  & 120                 & 24\%\\  
	\texttt{ITML840106} & loam            & 0.60 & 0.30                                  & 110                 & 27\%\\  
	\texttt{ITML840107} & silty clay loam & 0.25 & 0.60                                  & 105                 & 8\%\\
\end{tabular}
\end{table}

%% file: tables/practices.tex
\begin{table}
\caption{Maize nitrogen fertilizer recommendations for maize in Koutiala, Southern Mali, that were considered in the virtual experiment. Whether or not rainfall and plant nitrogen stress were considered as factors for the fertilizer recommendation is indicated by Yes or No. `\texttt{NSTRES}' stands for plant nitrogen stress and ‘DAP’ for days after planting.
\label{tab:practices}}
\begin{tabular}{p{1cm}p{2cm}p{1.5cm}p{1.5cm}p{1.5cm}p{2cm}p{2cm}p{2cm}}
\textbf{index} &  \textbf{max amount applied  (kgN/ha)} & \textbf{max applications} & \textbf{rainfall threshold} & \textbf{\texttt{NSTRES} threshold} & \textbf{15 DAP  N (kgN/ha)} & \textbf{30 DAP N (kgN/ha)} & \textbf{45 DAP N (kgN/ha)}\\ \midrule
0 & 135  & 2  & No   & No   & 15  & 120 & 0\\
1 & 135  & 2  & No   & Yes  & 15  & 120 & 0\\
2 & 135  & 2  & Yes  & No   & 15  & 120 & 0\\
3 & 135  & 2  & Yes  & Yes  & 15  & 120 & 0\\
4 & 135  & 3  & No   & No   & 15  & 60  & 60\\
5 & 135  & 3  & No   & Yes  & 15  & 60  & 60\\
6 & 135  & 3  & Yes  & No   & 15  & 60  & 60\\
7 & 135  & 3  & Yes  & Yes  & 15  & 60  & 60\\
8 & 70   & 2  & No   & No   & 23  & 0  & 47\\
9 & 180  & 3  & No   & No   & 60  & 60 & 60\\
\end{tabular}
\end{table}

%% file: algorithms/bcb.tex
\begin{algorithm}
	\caption{Simplified pseudo-code of \texttt{BCB}. \label{algo:bcbPseudoCode}}
    \begin{algorithm2e}[H]
    	\For{$\text{\emph{fertilizer practice }} k \in \{1, \cdots, K\}$}{
        	\nosemic Add maximum observable value to the results of fertilizer practice $k$ \tcp*{prior to any experiments}\;
    	}
    	\For{$\text{\emph{season }} t \in \{1, \cdots, T\}$}{
    	    \For{$\text{\emph{farmer }} f \in \{1, \cdots, n\}$}{
    	        \For{$\text{\emph{fertilizer practice }} k \in \{1, \cdots, K\}$}{
    	            \nosemic Re-weight the rewards of the fertilizer practice $k$ with random weights sampled from a Dirichlet distribution \citep{everitt2002cambridge}\;
    	            Score practice $k$ with a noisy empirical measure of the CVaR at level $\alpha$ of practice $k$ from the re-weighted rewards\;
    	        }
    	    \nosemic Recommend to the farmer $f$ the fertilizer practice with the maximum score\;
    	    }
    	    \nosemic Collect and store all results of the season for all fertilizer practices\;
        }
    \end{algorithm2e}
\end{algorithm}

%% file: corpus/results.tex
\subsection{Simulated responses to nitrogen fertilizer practices}
\label{sec:yieldDists}
Table~\ref{tab:optimalpractices} provides the statistics of the optimal nitrogen fertilizer practices for each soil type (Table~\ref{tab:soils}), i.e.\@ for each cohort, and Figure~\ref{fig:responses} in Supplementary Materials shows the distribution of grain yield, ANE and YE responses. All responses showed values within the expected ranges for the considered growing conditions, with an average grain yield varying from 3125 kg/ha for a sandy soil with low fertility (\texttt{ITML84105}) up to 3945 kg/ha for a loamy soil (\texttt{ITML84106}). When a applying the most promising fertilization strategies, on average the YE (i.e. yield gain compared to the reference) for farmers ranged from 1200 kg/ha to 1800 kg/ha, and the $\cvar{30\%}(\text{YE})$ (i.e.\@ the mean crop YE of the 30\% worst observable years) from 500 kg/ha to 1032 kg/ha.

There was no simple parametric assumption that could be made about YE, such as its probability distribution to be Gaussian (e.g.\ practice 5 in Figure~\ref{fig:YEDist04}). The thicker left tails for e.g. fertilizer practices 4 and 0 or the bi-modality of YE for practices 6 and 7 (Figure~\ref{fig:YEDist04}), further supported the use of the CVaR as a relevant risk measure. Indeed, the CVaR is most relevant for asymmetric and irregularly shaped distributions, such as thick-tailed or multi-modal distributions. For all soils, the optimal nitrogen fertilizer practices were either nitrogen fertilizer practice 0 or 8 i.e.\@ nitrogen practices without threshold dependent top-dressing, and with a single nitrogen top-dressing application (Table~\ref{tab:optimalpractices}).

The nitrogen fertilizer practices had different responses for the different soil types in terms of the grain yield and ANE (and consequently YE), and ranking of the practices was inconsistent across the soil types (Figure~\ref{fig:responses}). For instance, for the soil \texttt{ITML840104} (silt clay loam of medium fertility), fertilizer practices 0 to 4 had similar YE (Figure~\ref{fig:YEDist04}). For the soil \texttt{ITML840105} (silt clay loam of low fertility), practices 0, 1 and 4 were substantially better than practices 2 and 3 (Figure~\ref{fig:YEDist05}).

Threshold-based fertilizer practices behaved inconsistently across the soil types. As an example, for the bi-modal YE distribution of the fertilizer practice 1, most of the probability density was concentrated around 0 kg/ha for the soil \texttt{ITML840104} (Figure~\ref{fig:YEDist04}) and around 1800 kg/ha for the soil \texttt{ITML840105} (Figure~\ref{fig:YEDist05}). For the soil \texttt{ITML840104} and practice 1, YE were mostly found around 0 kg/ha because most of the seasons, the nitrogen-stress threshold of 0.2 was not reached, and consequently no top-dressing occurred (Table~\ref{tab:practices}). In such cases, only a basal-dressing of 15 kg N/ha was applied, instead of a total of 135 kg N/ha when the top-dressing was triggered. Consistently, for the same soil and fertilizer practice, the probability density of grain yield was concentrated around the low value of 1000 kg/ha (Figure~\ref{fig:yieldDist04}). On the other hand, with the soil \texttt{ITML840105}, most of the seasons, the nitrogen-stress threshold of 0.2 was reached and practice 1 applied both basal and top-dressing. This corresponded to YE mostly found around 1800 kg/ha (Figure~\ref{fig:YEDist05}), and the corresponding grain yields were mostly found around 4000 kg/ha (Figure~\ref{fig:yieldDist05}).

\input{tables/practiceStatistics}
\subsection{Identification of best fertilizer practices}
\label{sec:resultsBanditIdentification}
\noindent In Section~\ref{sec:empiricalPerformances} and~\ref{sec:regretPerformances}, we present respectively a direct measure of empirical performances of the nitrogen fertilizer practice identification strategies (see Section~\ref{sec:empiricalMeasure}), and the regret as a proxy measure, both for the farmer population average and the individual farmer regret distribution (see Section~\ref{sec:syntheticMeasure}). Section~\ref{sec:samplingVisualization} provides a visual comparison of nitrogen fertilizer recommendations following respectively the \texttt{BCB} and \texttt{ETC-5} identification strategies.

\subsubsection{Sampling visualization}
\label{sec:samplingVisualization}
Figures~\ref{fig:samplingProportions} provides the average frequency with which the fertilizer practices were selected by the identification strategies, from the beginning of the experiment to time $T$, for soils \texttt{ITML840105} and \texttt{ITML840101}. For the soil \texttt{ITML840105}, respectively for the \texttt{BCB} and \texttt{ETC-5} strategies. After 20 years, \texttt{BCB} had selected the fertilizer practice 8, which was the optimal one (see Table~\ref{tab:optimalpractices}), with an average proportion of 50\%. The proportions of the optimal practice continuously increased from year 2 onwards (Figure~\ref{fig:BCTVSsampling05}). During the first 5 years, \texttt{ETC-5} uniformly sampled all fertilizer practices (Figure~\ref{fig:ETC3sampling05}), thus inducing potentially high losses for farmers. The proportion of the optimal practice started to increase from year 5 onwards. After year 20, \texttt{ETC-5} sampled the optimal practice with an average proportion of 31\%. For soil \texttt{ITML840101}, results are more contrasted. After year 20, both \texttt{BCB} has sampled the optimal strategy, which was fertilizer practice 0 (see Table~\ref{tab:optimalpractices}) with an average proportion of 27\% (Figure~\ref{fig:BCTVSsampling01}) and \texttt{ETC-5} (Figure~\ref{fig:ETC3sampling01}) with an average proportion of 26\%. Note that in Figures~\ref{fig:BCTVSsampling01} and~\ref{fig:ETC3sampling01}, the color differences are almost not perceptible for nitrogen fertilizer practices 0, 1 and 4, because all three practices showed similar performances. In Sections~\ref{sec:empiricalPerformances} and~\ref{sec:regretPerformances}, we provide the results of statistics that account for all cohorts, i.e.\@ soils.

\begin{figure}
	\centering
	\begin{subfigure}[t]{0.5\textwidth}
		\centering
		\includegraphics[width=\textwidth]{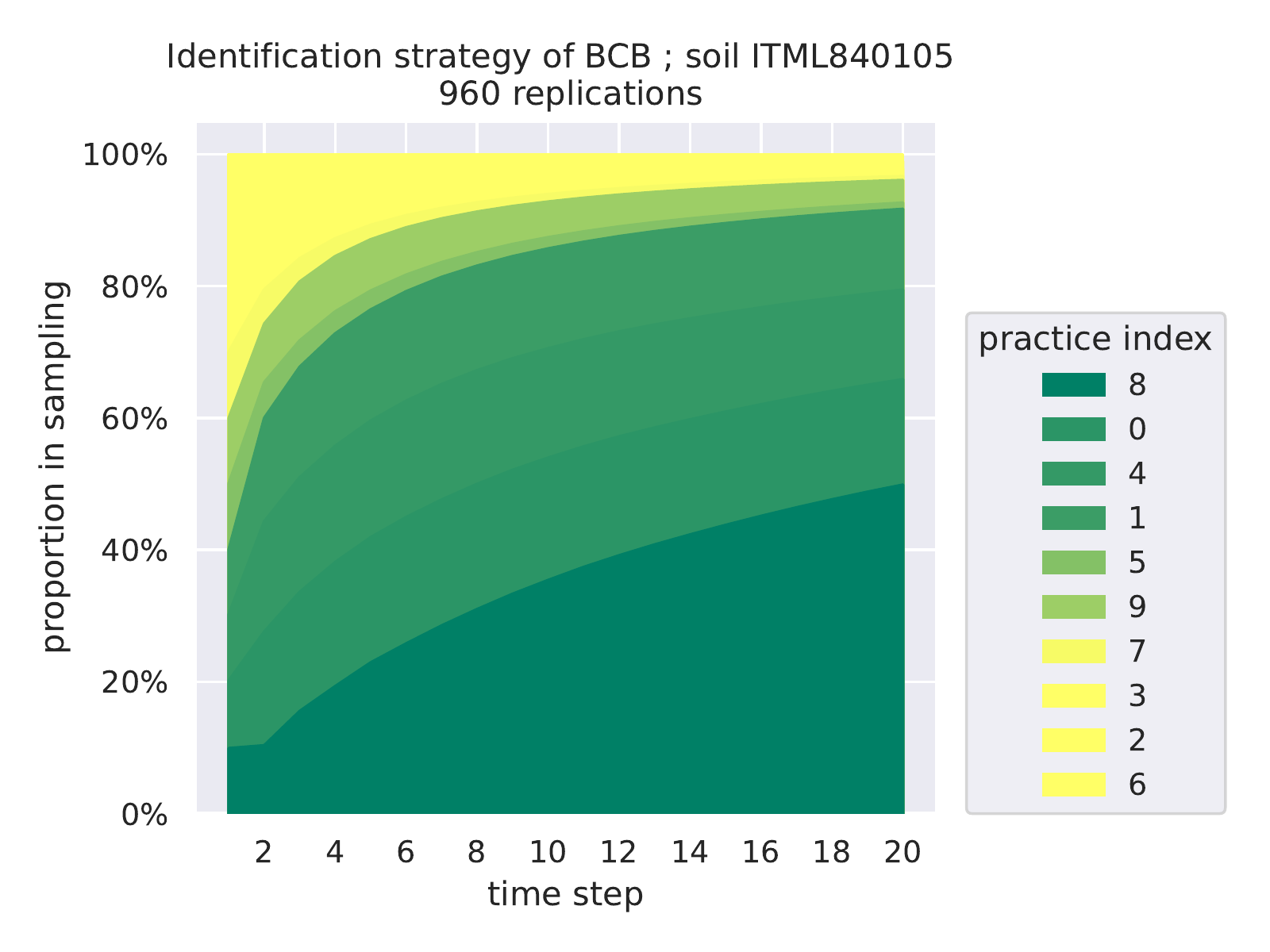}
		\caption{\texttt{BCB} sampling proportions for soil \texttt{ITML840105}. \label{fig:BCTVSsampling05}}
	\end{subfigure}%
	\begin{subfigure}[t]{0.5\textwidth}
		\centering
		\includegraphics[width=\textwidth]{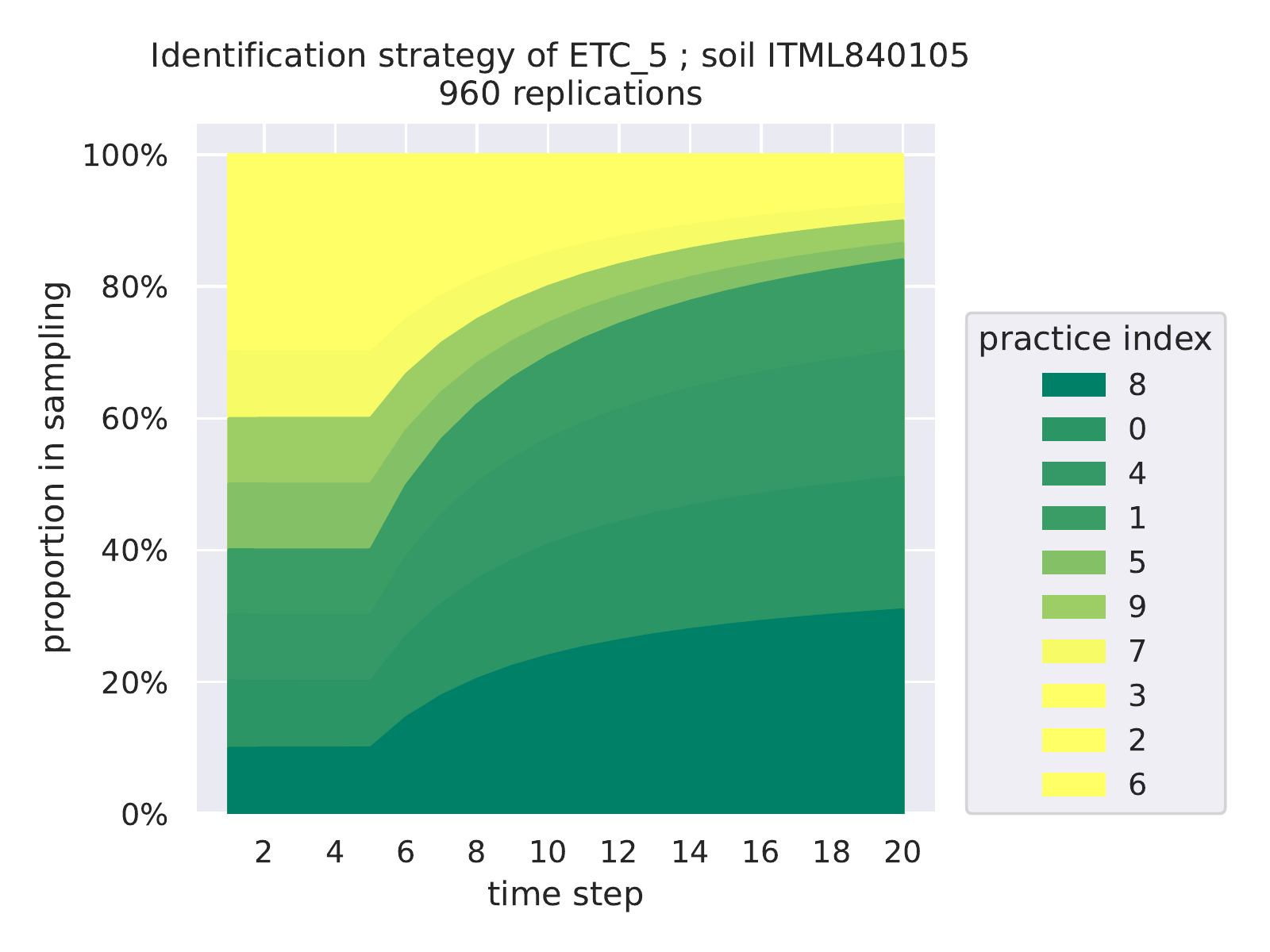}
		\caption{\texttt{ETC-5} sampling proportions for soil \texttt{ITML840105}. \label{fig:ETC3sampling05}}
	\end{subfigure}%
	\hfill
	\begin{subfigure}[t]{0.5\textwidth}
		\centering
		\includegraphics[width=\textwidth]{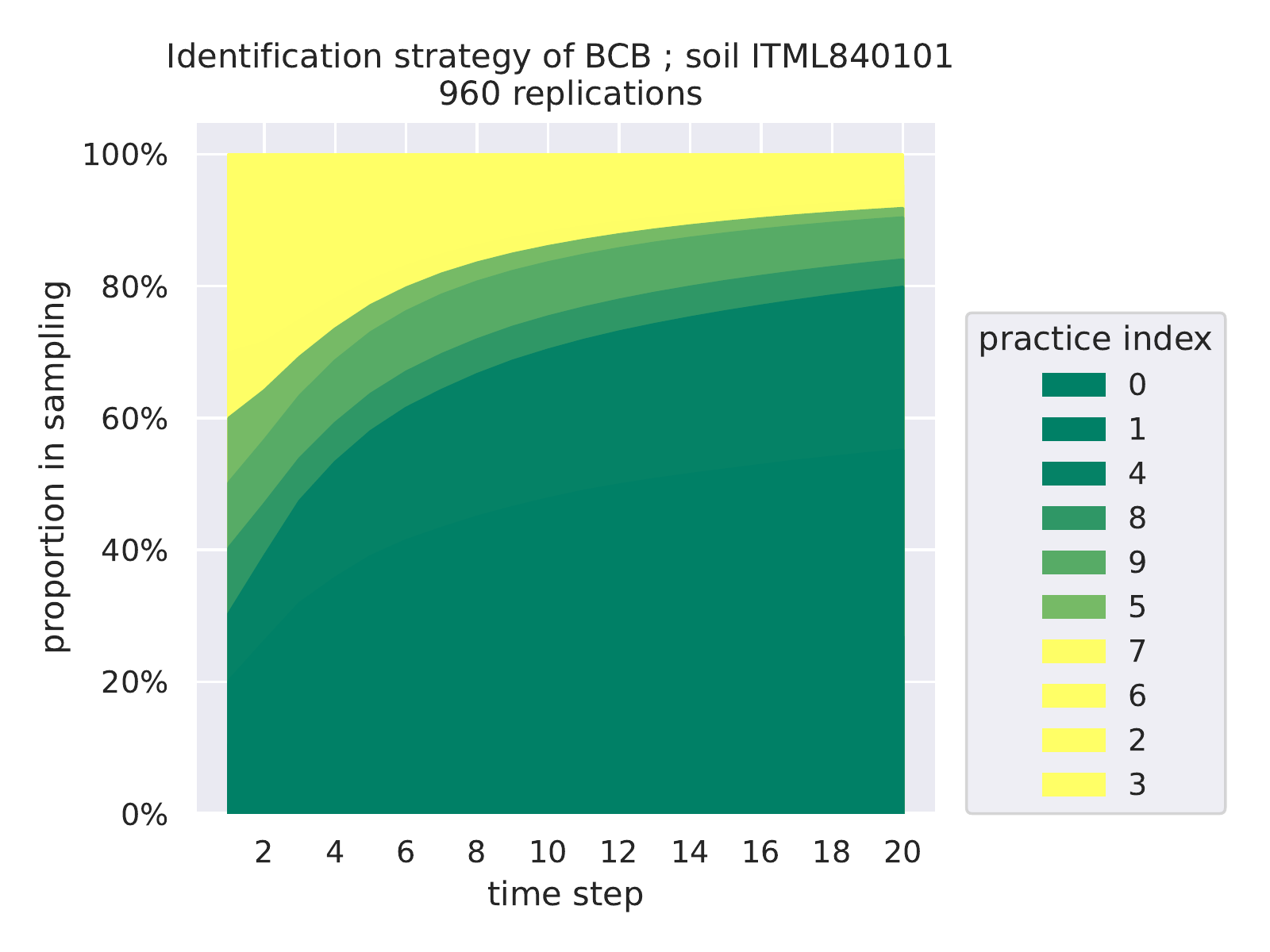}
		\caption{\texttt{BCB} sampling proportions for soil \texttt{ITML840101}. \label{fig:BCTVSsampling01}}
	\end{subfigure}%
	\hfill
	\begin{subfigure}[t]{0.5\textwidth}
		\centering
		\includegraphics[width=\textwidth]{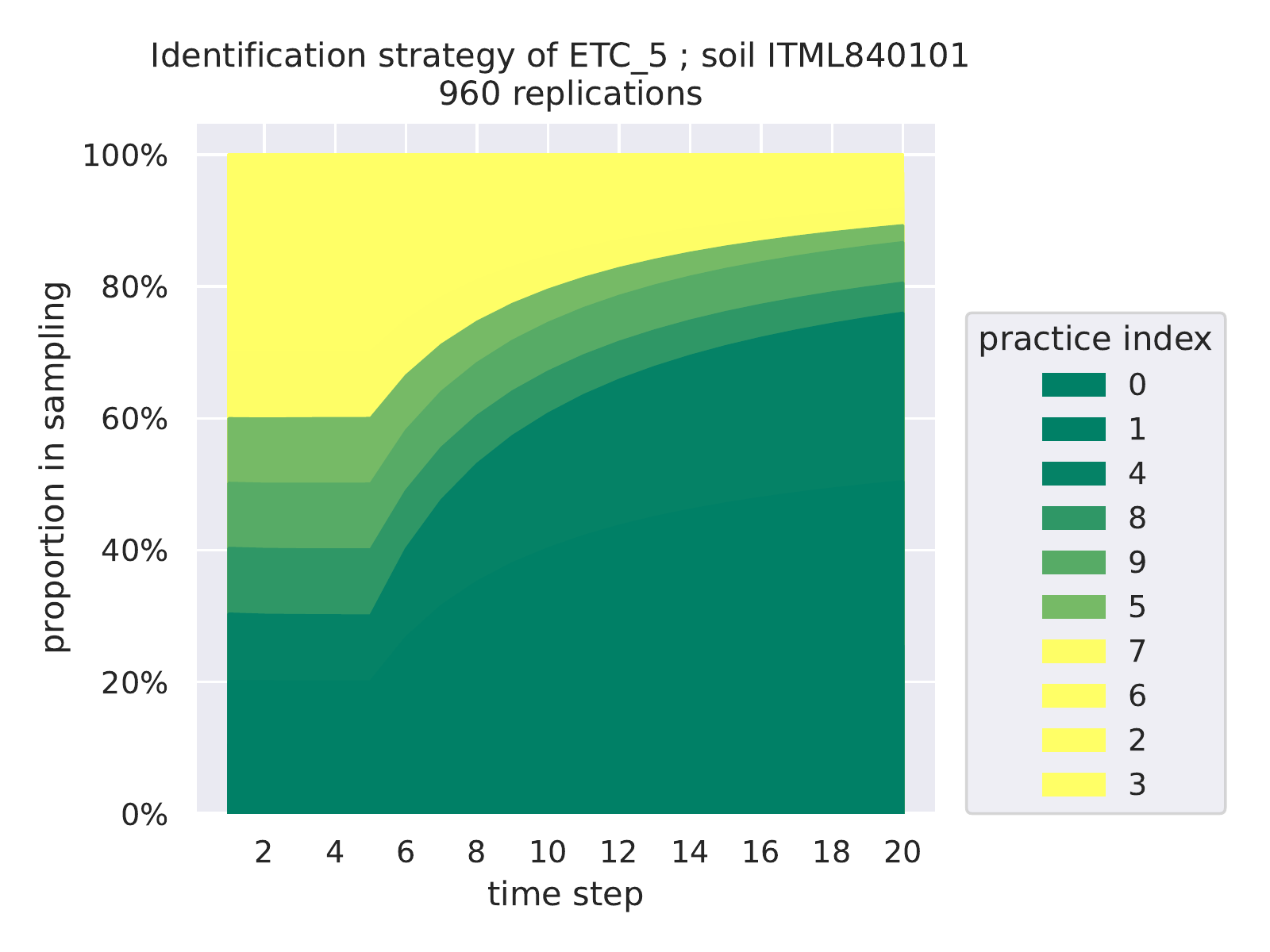}
		\caption{\texttt{ETC-5} sampling proportions for soil \texttt{ITML840101}. \label{fig:ETC3sampling01}}
	\end{subfigure}%
\caption{Averaged sampling proportions for soils \texttt{ITML840105} and \texttt{ITML840101}, $T=20$ years. 960 replications of the whole experiment were done. The fertilizer practices are ordered according to the true Conditional Value-at-Risk at level 30\% (CVaR) of their Yield Excess (YE) with $\text{ANE}_{\text{ref}}$=15 kg grain/kg N ; the greener the color, the better a fertilizer practice is. \label{fig:samplingProportions}}
\end{figure}

\subsubsection{Direct measure of performances of identification strategies}
\label{sec:empiricalPerformances} Figure~\ref{fig:populationRegret} represents the evolution of the $\cvar{30\%}(\text{YE})$ for all cohorts trough the years (Equation~\ref{eq:farmerObjEst}). On average, farmers following the nitrogen fertilizer recommendations based on the \texttt{BCB} strategy had higher empirical CVaR at 30\% of YE than farmers following those from ETC strategies, from the second year of the experiment onwards (Figure~\ref{fig:populationEmpiricalCvar}). The difference in performance between \texttt{BCB} and \texttt{ETC} is high during the initial years. For instance, at year 4, farmers following recommendations from the \texttt{BCB} identification strategy had a CVaR at 30\% of YE of 318 kg/ha, compared to 168 kg/ha (47\% less than \texttt{BCB}) and 74 kg/ha (77\% less than \texttt{BCB}) for farmers following the recommendations respectively from the \texttt{ETC-3} and the \texttt{ETC-5} identification strategies. \texttt{BCB} allowed to identify faster the optimal fertilizer practices and consequently further avoided low crop yield outcomes compared to ETC strategies. ETC strategies were adversely affected by their exploration phases during which all fertilizer practices were equiproportionally tested. In contrast, \texttt{BCB} had a continuously increasing empirical CVaR, for the whole duration of the experiment.

\begin{figure}
    \centering
    \includegraphics[width=.65\textwidth]{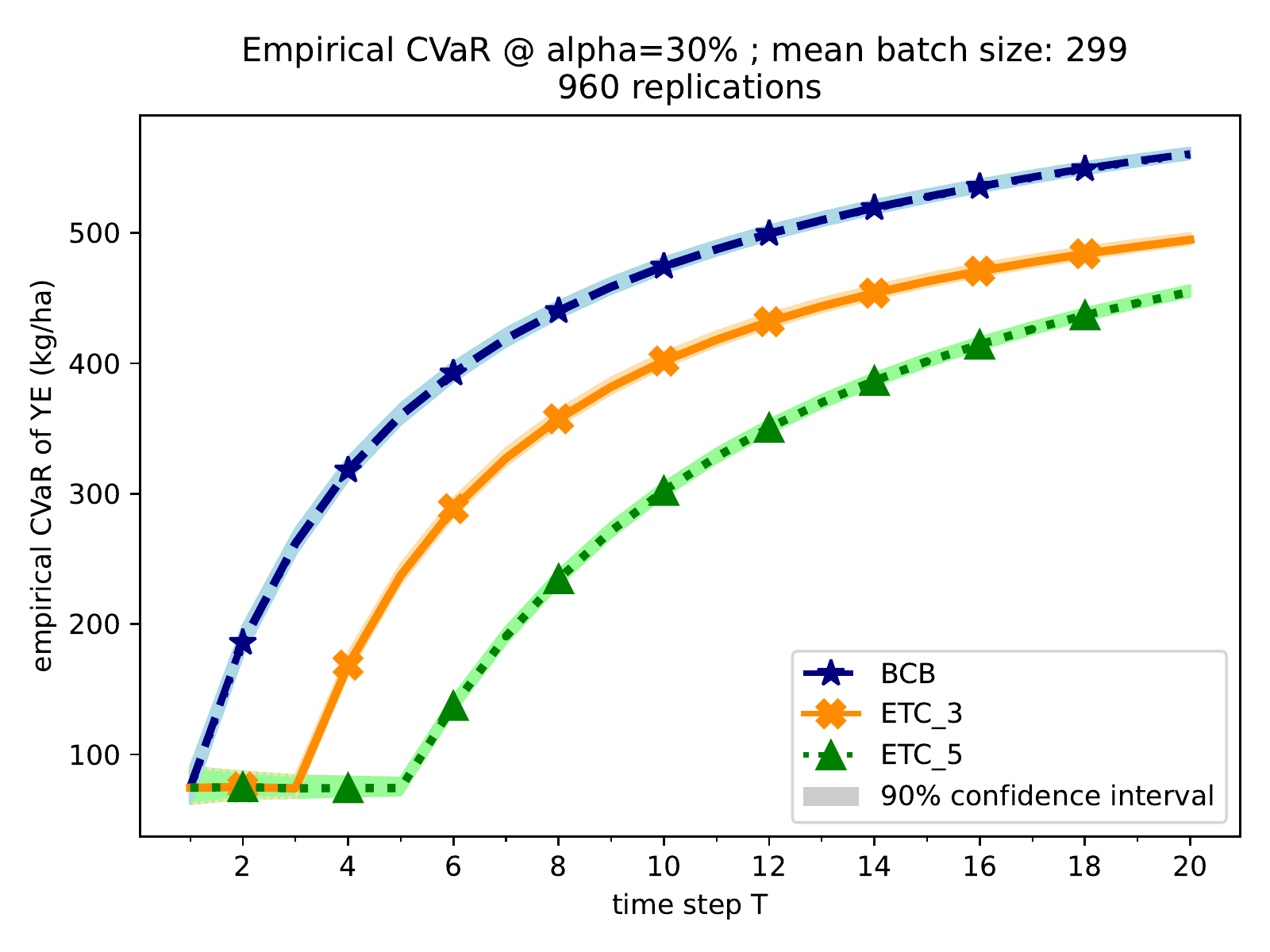}
    \caption{Empirical conditional Value-at-Risk (CVaR)  at level 30\% (CVaR) of maize yield excesses (YE) between $T=0$ and the considered $T$ ; $\text{ANE}_{\text{ref}}=15$ kg grain/kg N. 960 replications of the whole experiment were done. One time step $T$ is one year ; `mean batch size' is the number of farmers who have volunteered to participate at the trials, averaged over all years and all replications. Confidence intervals were computed following \citet{thomas2019concentration}.}
    \label{fig:populationEmpiricalCvar}
\end{figure}

\subsubsection{Regret}
\label{sec:regretPerformances}
\paragraph{Mean cumulated regret of the farmer population}
Figure~\ref{fig:populationRegret} represents the evolution of the mean regret for all cohorts trough the years (Equation~\ref{eq:regretAllCohorts}). For $\alpha=30\%$, \texttt{BCB} identification strategy outperformed \texttt{ETC} strategies, regardless of the number of years during which the strategy was applied. The difference in performance between \texttt{BCB} and \texttt{ETC} increases for the whole duration of the experiments. After 20 years, farmers following recommendations from \texttt{BCB} identification strategy experiences a mean cumulated regret of 2400 kg/ha, compared to 3385 kg/ha (41\% more than \texttt{BCB}) and 3701 kg/ha (54\% more than \texttt{BCB}) for farmers following the recommendations respectively from the \texttt{ETC-3} and \texttt{ETC-5} strategies. Consequently, farmers following \texttt{BCB} recommendations accumulated less regret compared to farmers following \texttt{ETC} recommendations. Furthermore, the variance of the cumulated regret (due to all different weather series in the experiments, for each season and each field trial, and the variability in cohorts each year) was smaller for \texttt{BCB} than for \texttt{ETC}, confirming that \texttt{BCB} strategy was more robust (see quantile ranges in Figure~\ref{fig:populationRegret}) for this decision problem.

\begin{figure}
    \centering
    \includegraphics[width=.65\textwidth]{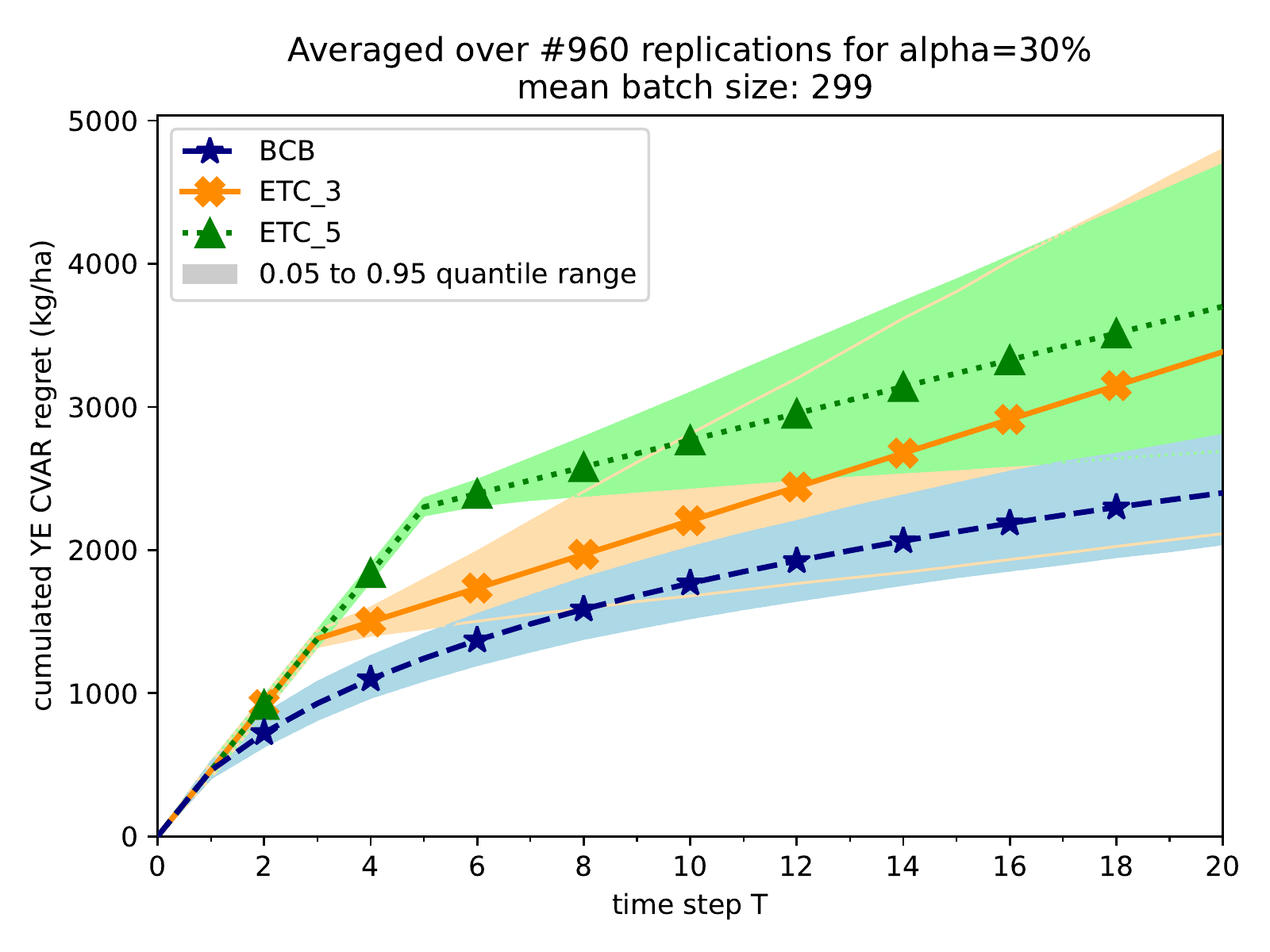}
    \caption{Mean cumulated regret of population, for the Conditional Value-at-Risk (CVaR) at level 30\%  of Yield Excess (YE); $\text{ANE}_{\text{ref}}=15$ kg grain/kg N. The cumulated cumulated regret is averaged over the farmers' population, between $T=0$ and the considered $T$. 960 replications of the whole experiment were done. One time step $T$ is one year, `mean batch size' is the number of farmers who have volunteered to participate in the trials, averaged over all years and all replicates. \label{fig:populationRegret}}
\end{figure}

\paragraph{Individual cumulated regret distribution}
\texttt{BCB} prevented farmers from accumulating large individual cumulated regret during the participatory identification of the group (Figure~\ref{fig:individualRegret}): individual cumulated regrets for \texttt{BCB} were distributed towards lower values than for \texttt{ETC} strategies. With \texttt{BCB}, almost no individual cumulated regret was greater than 7.5 t/ha after 20 years, as opposed to \texttt{ETC} strategies. Consequently, \texttt{BCB} allowed a fairer sharing of identification mistakes in the population of farmers than \texttt{ETC} strategies.

\begin{figure}
	\centering
	\includegraphics[width=.65\textwidth]{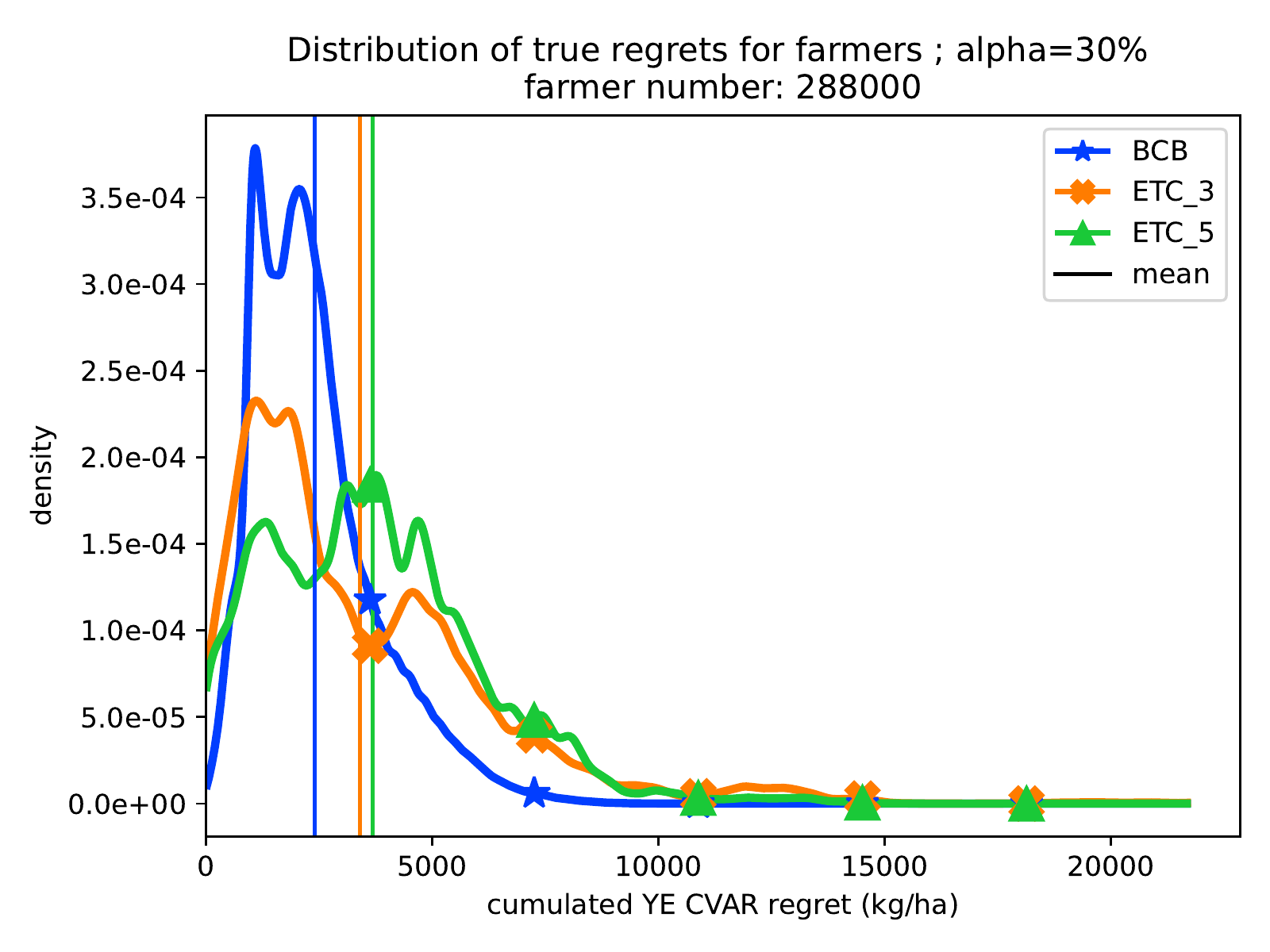}
	\caption{Distribution of individual cumulated regret after $T=20$ years for Conditional Value-at-Risk at level 30\% (CVaR) of the yield excess (YE) ; $\text{ANE}_{\text{ref}}=15$ kg grain/kg N. The total number of farmers corresponds to a group of 300 farmers, with 960 replications of the whole experiment.}
	\label{fig:individualRegret}
\end{figure}

\subsubsection{Sensitivity analysis}
\label{sec:sensitivityAnalysis}
In Section~\ref{sec:additionalExpBatchBandit} of Supplementary Materials, we present the same results than Sections~\ref{sec:empiricalPerformances} and~\ref{sec:regretPerformances} for higher CVaR levels of  $\alpha=50\%$ and $\alpha=100\%$. The CVaR with the latter level recovers the usual expectation. For $\alpha=50\%$, \texttt{BCB}  showed similar performance than for $\alpha=30\%$. For $\alpha=100\%$, \texttt{ETC-3} was the best performer, \texttt{BCB}  and \texttt{ETC-5} performed similarly. Nonetheless, \texttt{BCB}  showed a smaller variance than both \texttt{ETC-3} and \texttt{ETC-5}. The theoretical performance guarantee is presented in Section~\ref{sec:demoBatchBandit} of Supplementary Materials.

%% file: tables/practiceStatistics.tex
\begin{table}
\caption{Statistics of the optimal nitrogen fertilizer practices for each of the soil types presented in Table~\ref{tab:soils}. For the corresponding optimal nitrogen fertilizer practice $\pi^{*}$, we define $\text{N}^{\pi^{*}}$: quantity of nitrogen fertilizer applied; $\cvar{30\%}(X)$: conditional Value-at-Risk of $X$ of level $30\%$ (Section~\ref{sec:YEandCVaR}); $\bar{X}$: mean value of $X$; $\text{Y}^{\pi^{*}}$: maize grain yield; $\text{ANE}^{\pi^{*}}$: Agronomic Nitrogen use Efficiency; $\text{YE}^{\pi^{*}}$: Yield Excess (Section~\ref{sec:YEandCVaR}); parentheses indicate standard deviations. \label{tab:optimalpractices}}
\begin{tabular}{p{1.7cm}p{.5cm}p{1.6cm}p{2cm}lp{2cm}p{2.2cm}l}
soil & $\pi^*$ & $\bar{\text{N}}^{\pi^{*}}$ & $\cvar{30\%}(\text{Y}^{\pi^{*}})$ & $\bar{\text{Y}}^{\pi^{*}}$ & $\bar{\text{ANE}}^{\pi^{*}}$  & $\cvar{30\%}(\text{YE}^{\pi^{*}})$ & $\bar{\text{YE}}^{\pi^{*}}$ \\
& & (kg/ha) & (kg/ha) & (kg/ha) & (kg/kg) & (kg/ha) & (kg/ha) \ \\ \midrule
\texttt{ITML840101} & \textbf{0}            & 120.0 (1.0)      & 3091             & 3874 (666)       & 30.0  (5.4) & \textbf{1032}    & 1795 (651)    \\
\texttt{ITML840102} & \textbf{8}            & 69.8  (4.0)      & 2391             & 3150 (653)       & 33.2  (7.5) & \textbf{652}     & 1270 (529)    \\
\texttt{ITML840103} & \textbf{8}            & 70.0  (0.4)      & 2539             & 3152 (526)       & 34.4  (6.8) & \textbf{808}     & 1356 (475)    \\
\texttt{ITML840104} & \textbf{8}            & 69.9  (2.7)      & 2533             & 3339 (682)       & 31.7  (8.1) & \textbf{500}     & 1169 (565)    \\
\texttt{ITML840105} & \textbf{8}            & 70.0  (1.2)      & 2467             & 3127 (570)       & 34.2  (7.3) & \textbf{757}     & 1346 (508)    \\
\texttt{ITML840106} & \textbf{0}            & 120.0 (1.2)      & 3132             & 3945 (695)       & 28.9  (5.5) & \textbf{900}     & 1667 (660)    \\
\texttt{ITML840107} & \textbf{8}            & 69.9  (2.7)      & 2472             & 3247 (659)       & 32.5  (8.0) & \textbf{565}     & 1226 (559)  
\end{tabular}
\end{table}

%% file: corpus/discussion.tex
\subsection{Benefits from an adaptive identification strategy.}
\label{sec:benefAdaptDiscussBatchBandit}
\paragraph{Practical perspective}In multi-year multi-location on-farm trials,  participating farmers simultaneously conduct field experiments with crops over multiple seasons to compare e.g. crop management practices \citep[e.g.\@][]{naudin2010impact, baudron2012comparative, falconnier2016unravelling}. After a given number of years, results (often in terms of crop yields) are typically analyzed using mixed linear models \citep{laird1982random}, to take into account the design of an experiment with repeated measures, such as random effects associated with fields and farms. Best crop management practices are then identified based on this statistical analysis. In our simulated nitrogen fertilizer practice decision problem, we approximated multi-year on-farm trials with the \texttt{ETC} intuitive identification strategy. Both replicated on-farm trials and \texttt{ETC} consist of an exploration phase of a fixed duration (data collection), followed by an exploitation phase (application of the best identified practice after analysis of collected data). Consequently, both replicated on-farm trials and \texttt{ETC} can be considered as non-adaptive identification strategies: before the end of the exploration phase, the intermediary results are not exploited to gradually refine the experimental setup. In contrast, \texttt{BCB} refines its the recommendations every year, based on the results observed so far. The better a crop management option, the more its proportion in recommendations should increase with time. From a farmer's perspective, this mean that the probability of highly sub-optimal recommendation decreases with time, as opposed to non-adaptive identification strategies during the exploration phase, which equi-proportionally recommend all crop management practices. Because with the bandit-algorithm-based identification strategy yield losses are reduced in most cases, compared to the non-adaptive identification strategies, the cost of the identification of best management practices is likely to be reduced for the farmers. Another common method to generate crop management recommendation consists in the use of calibrated crop models and scenario analyses \citep[e.g.\@][]{huet2022coping}. This method can be complementary to the bandit-based approach. Candidate crop management practices can first be determined based on crop modeling results for the considered conditions, and then best crop management can be identified with the bandit algorithm..

\paragraph{Theoretical perspective}\texttt{ETC} is theoretically proven to be a sub-optimal identification strategy without a calibration of the duration of the exploration phase that requires unavailable strong prior knowledge on the complexity of the decision problem \citep[Chapter~6]{lattimore2020bandit}. In numerical experiments, for $\alpha=100\%$, \texttt{ETC-3} best performed, probably because with these particular YE distributions and size of farmer group, 3 years of exploration was an optimal number. A slight change in the decision problem may induce that 3 or 5 years of exploration phase are no longer optimal (e.g.\@ changing $\alpha$ to $30\%$ or $50\%$). More generally, prior to an experiment, there is no guarantee than an arbitrate number of years of exploration of \texttt{ETC} will be optimal, and consequently there are no guarantees about the performance \texttt{ETC}, as opposed to \texttt{BCB} (see theoretical results in Section~\ref{sec:demoBatchBandit}). The main benefit of \texttt{BCB} over \texttt{ETC} is that it does not require a choice on parameters that require prior knowledge that is \textit{a priori} not available. \texttt{BCB} neither requires strong assumptions about probability laws of reward distributions, as opposed to other common bandit algorithms. The only requisite for \texttt{BCB} is the knowledge of the maximum observable reward. In agronomy, such knowledge is easily available with expert knowledge: for instance, considering yield as reward, an expert can easily estimate a yield potential in the given crop growing conditions, for instance through modeling \citep{affholder2013yield}.

\subsection{Performances of fertilizer practices}
\label{sec:fertiPolDiscussBatchBandit}
For all soils, no optimal nitrogen fertilizer practice was threshold-based, nor shown split top-dressing. This does not discard a potential benefit from the threshold-based fertilizer practices, or split top dressing. Model simulations revealed that the effect of the nitrogen-stress or rainfall threshold depended on each soil, and the effect of the thresholds was not easy to anticipate. Consequently, the definition of the set of candidate fertilizer practices to explore must be carefully selected within the vast possible combination of practice parameters, e.g.\@ application timing, rates, thresholds or number of split. In the experiments, the optimal values of practice parameters were not adjusted, because our focus was on designing a better generic identification method, rather than on designing refined parametrized fertilizer practices. For an application in real field conditions, we recommend these parameters to be adjusted based on simulations (see Section~\ref{sec:benefAdaptDiscussBatchBandit}) and/or on prior small test plots. More generally, the design of fertiliser practices must include experts, local extensionists and farmers themselves \citep{cerf2006outils, hochman2011emerging}. For instance, the maximum quantity of nitrogen fertilizer a farmer can apply may depend on the availability of fertilizer in the local market, and on the economic situation of each farmer.

\subsection{Definition of farmers' objective}
\label{sec:discussFarmersObj}
We defined the farmers' objective as maximizing the CVaR at level $\alpha=30$ of the YE with $\text{ANE}_\text{ref}=15$ kg grain/kg N. This quantity is interpretable as it represents a yield gain compared to a reference fertilizer practice, and it is easily calculable. The value of $\alpha$ allows to adjust the risk aversion level for a cohort of farmers. The value of $\text{ANE}_\text{ref}$ defines an invariant economical and environmental trade-off which penalizes more or less the use of nitrogen fertilizer. Losses were defined as the expected performance difference between the best available nitrogen fertilizer practice, and the sub-optimal nitrogen fertilizer practices, in the face of the seasonal uncertainty.

However, we did not directly evaluate fertilizer practices by their economic return. Despite market risks being a reality, the economic return of maize nitrogen fertilization depends on many parameters changing through time, such as fertilizer subsidies, fertilizer market price, application costs, or harvest selling price. Because each year, the optimal nitrogen fertilizer practice is likely to change, such setting dramatically increases the complexity of the identification problem, and so does the required amount of data to identify best practices (we provide more details in Supplementary Materials, Section~\ref{sec:YEcomplement}). In any case, modelers should bear in mind the inherent limitations of the modeling of a farmer’s objective, which always remains a proxy \citep{mccown2002changing}.

\subsection{Limits and possible improvements}
\label{sec:LimitsDiscussBatchBandit}
In our simulated crop management decision problem,  we largely simplified the experimental structure of multi-year replicated field trials.  First, for all simulations, weather time series were independent and identically distributed. Such assumption is unlikely to be true in the real world. During the same year, weather spatial correlations can be high, for instance in case of extreme weather events \citep{tack2016influence}. Second, within the same cohort, all farmers were supposed to have exactly the same soil and cultivar, and to implement closely the fertilizer practice they were assigned. For real application, a farmer’s soil, site, year and other potential random effects should be properly considered. The bandit identification strategy we introduced should be extended to account for experimental structure and multiple factors at stake. For instance, contextual bandits \citep{lattimore2020bandit}, which would allow to share information between decision contexts (here, the cohorts), might offer solutions.

As another limit, in simulations, we considered climate to be the same during the 20 years of the experiment. Such hypothesis is improbable in real conditions  \citep[e.g.\@][]{traore2017modelling}. Nevertheless, as \citet{adam2020more} has shown based on simulations, in Mali, improving current crop management, in particular nitrogen fertilization, may compensate the long-term effects of climate change, while addressing the urgent necessity of closing yield gaps. For a decision problem perspective, if climate changes through time, then optimal practices are likely to change with time. Such problem can be formalized as a non-stationary bandit problem \citep{lattimore2020bandit}.  To handle non-stationary, BCB can be equipped with a sliding window \citep{garivier2011upper, baudry2021limited}.  This mechanism forces the bandit algorithm  to overlook observations older than a given number of years, which consequently must regularly re-evaluate all fertilizer practices. Such approach reiterates the recommendations formulated by \citet{adam2020more}: the bandit algorithm would handle climate change by regularly trying to improve current fertilizer practices.

%% file: corpus/conclusion.tex
We addressed the problem of the identification of best maize fertilizer practices, supported by virtual trials of a community of smallholder farmers. Our goal was to provide an identification strategy that minimized farmers’ yield losses occurring during field trials, compared to the ‘intuitive strategy’ which consisted of multi-year field trials with an equal proportion of each fertilizer practice tested each growing season. In simulated experiments mimicking the conditions of southern Mali as a case study, the bandit-based identification strategy we introduced showed better resulted in reducing farmers’ yield losses. We used model simulations to compare the identification strategies, but the bandit-based identification strategy does not depend on the simulations. This novel approach opens up new perspectives as an alternative to the usual multi-year on-farm trials. It can also complement the use of calibrated crop simulation models for the formulation of fertilizer recommendations. For instance, if parameterized fertilizer practices are explored, the choice of the parameters could be made based on prior simulations and/or small-scale field trials.

However, before confirming that the bandit-based identification strategy can be employed to identify best management practices in real conditions, several constraints must be addressed: (i) the structure of multi-year on-farm field trials with repeated measures and thus correlations, such that the effect, for the same year in a given region, of spatial correlation of weather series; (ii) the effect of climate change; (iii) the effect of farmer compliance, who may not strictly follow the fertilizer practice recommendations, which may induce extra noise. We briefly indicated how our approach can be extended to address the aforementioned constraints.

%% file: corpus/softwareAvailability.tex
\noindent All the numerical experiments in this paper are meant to be as reproducible as possible, and the code is open source. The Python code with the necessary packages, instructions and experimental data are provided in the following public GitLab repository: \url{https://gitlab.inria.fr/rgautron/batch-cvts/-/tree/master}. The simulations are performed with \texttt{gym-dssat} (\url{https://gitlab.inria.fr/rgautron/gym_dssat_pdi}), a modified version of the Decision Support System for Agrotechnology Transfer (DSSAT) software (\url{https://dssat.net/}).

%% file: appendices/maizeSimulations.tex
\paragraph{}The cultivation scenarios were based on the the conditions found in Southern Mali. The soils came from \citet{adam2020more} who compiled and supplemented with survey data the soils found in the literature for the location of Koutiala, Mali. The data of \citet{adam2020more} included soils' depth, texture, water capacity, bulk density, organic matter content, pH and initial mineral nitrogen content. Soil characteristics and proportions in the population were summarized in Table~\ref{tab:soils}, based on \citet{adam2020more}. During the simulations, the weather times series were generated using the WGEN weather model \citep[see][]{richardson1984wgen, soltani2003statistical}. WGEN had been calibrated on 40 year long historical daily weather records from a weather station located in N'Tarla found in \citet{ripoche2015cotton}, which was located about 20 km from Koutiala ; these historical weather records were the best available. The cultivars used in the simulation and its parametrization in DSSAT are presented in Table \ref{tab:cultivarsBatchBandit} ;  this cultivars comes with DSSAT default data and was representative of the cultivars used in Mali. The cultivars were already calibrated based on experiments carried out in Mali. The simulations were initiated on Day Of Year (DOY) 140 and the planting is automatically performed in a window ranging from DOY 155 to 185 ; we specified the parameters of the automatic planting with Table \ref{tab:automaticPlanting}. For each soil, the initial soil nitrogen content was set according to the values found in \citet{adam2020more}. The soil water content was set to crop lower limit, as a result of the end of the dry season at the usual planting dates. Because the simulations were initiated prior to planting date and because the weather was stochastically generated, the soil nitrogen mineral and water contents were uncertain at planting time. Each simulation was performed independently from the previous ones. At the beginning of the experiment, all the soils described in Table ~\ref{tab:soils} were randomly distributed amongst the initial group of farmers following the proportions provided in Table~\ref{tab:soils}. Figure~\ref{fig:responses} shows the simulated yield distributions for \texttt{ITML840104} and \texttt{ITML840105} soils. 

\input{figures/practiceDistributions}

\begin{table}[H]
\begin{center}
\caption{Maize cultivar parametrization in DSSAT \label{tab:cultivarsBatchBandit}}
	\begin{tabular}{ llllllll }
		\textbf{name} & \textbf{ecotype} & \textbf{P1} & \textbf{P2} & \textbf{P5} & \textbf{G2} & \textbf{G3} & \textbf{PHINT} \\ \midrule
		Sotubaka & \texttt{IB0001} & 300.0 & 0.520 & 930.0 & 500.0 & 6.00 & 38.90 \\
	\end{tabular}
\end{center}
\end{table}


\begin{table}[H]
\begin{center}
    \caption{Automatic planting parametrization in DSSAT.
    PFRST: Starting date of the planting window;
    PLAST: End date of the planting window;
    PH2OL: Lower limit on soil moisture for automatic planting;
    PH2OU: Upper limit on soil moisture for automatic planting; 
    PH2OD: Depth to which average soil moisture is determined for automatic planting;
    PSTMX: Maximum temperature of planting;
    PSTMN: Minimum temperature of planting.
    \label{tab:automaticPlanting}
    }
	\begin{tabular}{ ll}
		\textbf{PFRST} (DOY) & 155 \\
		\textbf{PLAST} (DOY) & 185 \\
		\textbf{PH2OL (\SI{}{\percent})}  & 40 \\
		\textbf{PH2OU (\SI{}{\percent})} & 100 \\
		\textbf{PH2OD} (\SI{}{\cm}) & 30 \\
		\textbf{PSTMX (\SI{}{\celsius}}) & 40 \\
		\textbf{PSTMN (\SI{}{\celsius}}) & 10 \\
	\end{tabular}
\end{center}
\end{table}

%% file: figures/practiceDistributions.tex
\newcommand{\widthFactorDists}{.35}
\begin{figure}
	\centering
		\begin{subfigure}[t]{\widthFactorDists \textwidth}
		\centering
		\includegraphics[width=\textwidth]{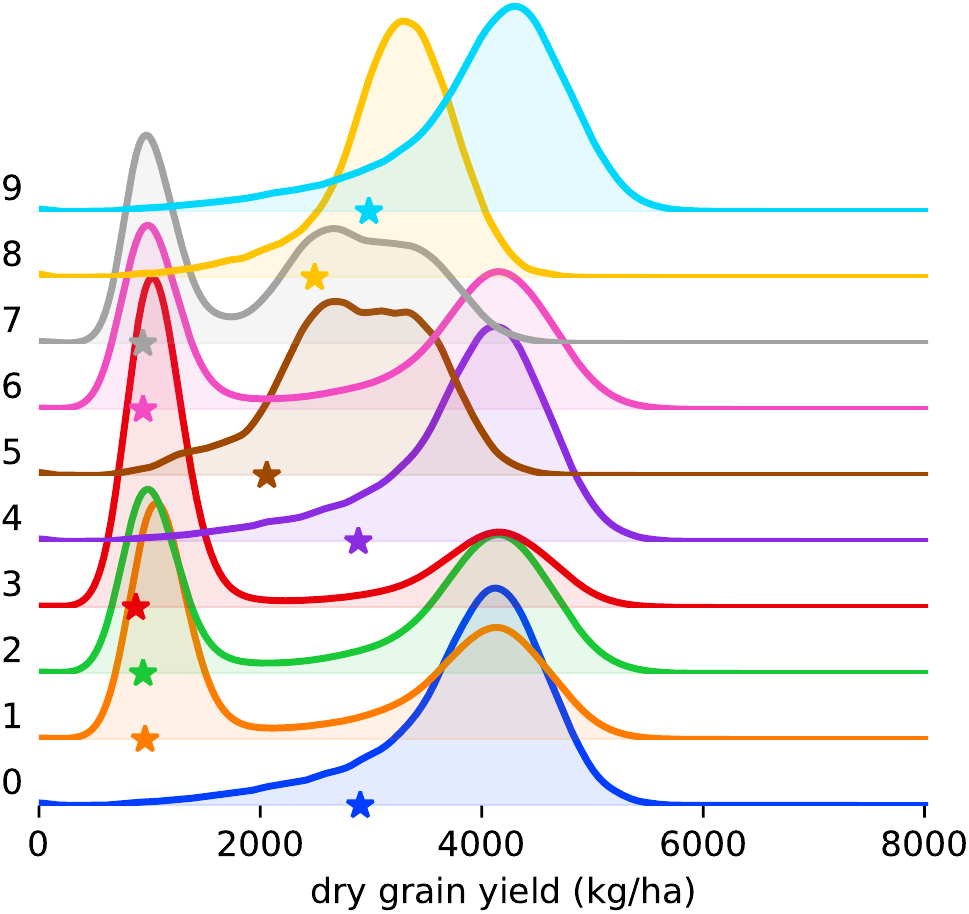}
		\caption{Yield distributions for soil \texttt{ITML840104}. Stars represent the CVaR at level 30\%.}
		\label{fig:yieldDist04}
	\end{subfigure}%
	\hfill%
	\begin{subfigure}[t]{\widthFactorDists \textwidth}
		\centering
		\includegraphics[width=\textwidth]{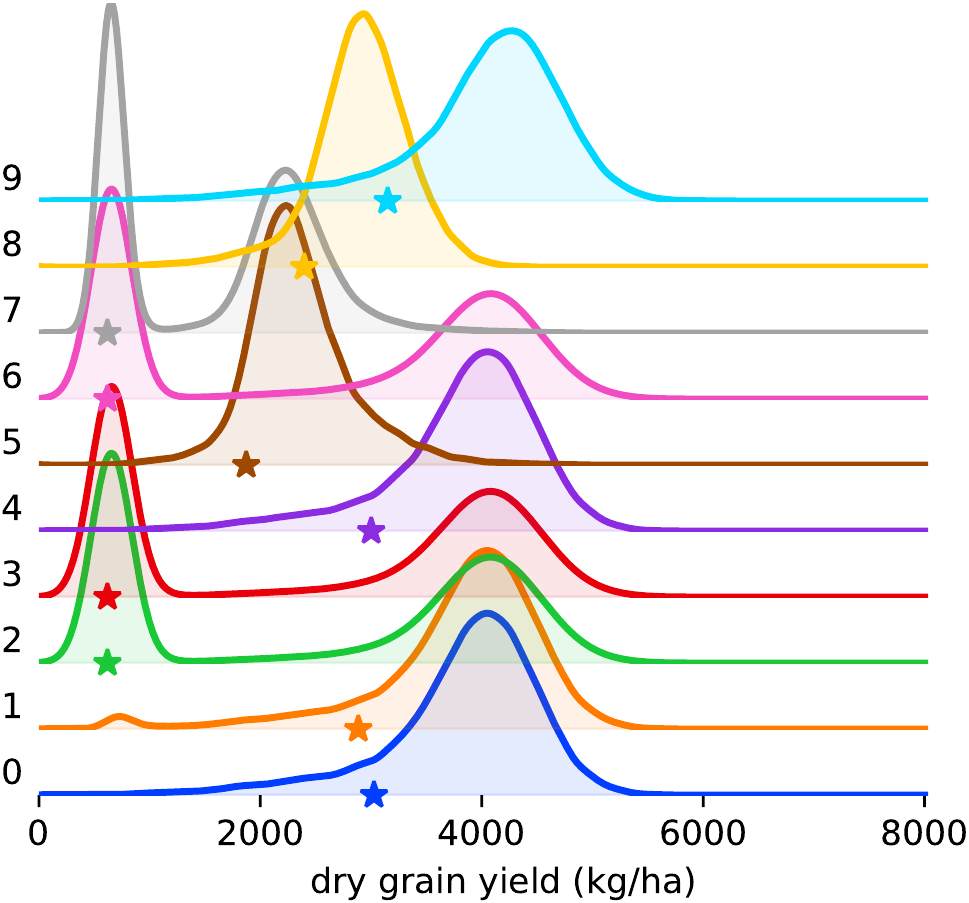}
		\caption{Yield distributions for soil \texttt{ITML840105}. Stars represent the CVaR at level 30\%.}
		\label{fig:yieldDist05}
	\end{subfigure}
	\hfill%
	\begin{subfigure}[t]{\widthFactorDists \textwidth}
		\centering
		\includegraphics[width=\textwidth]{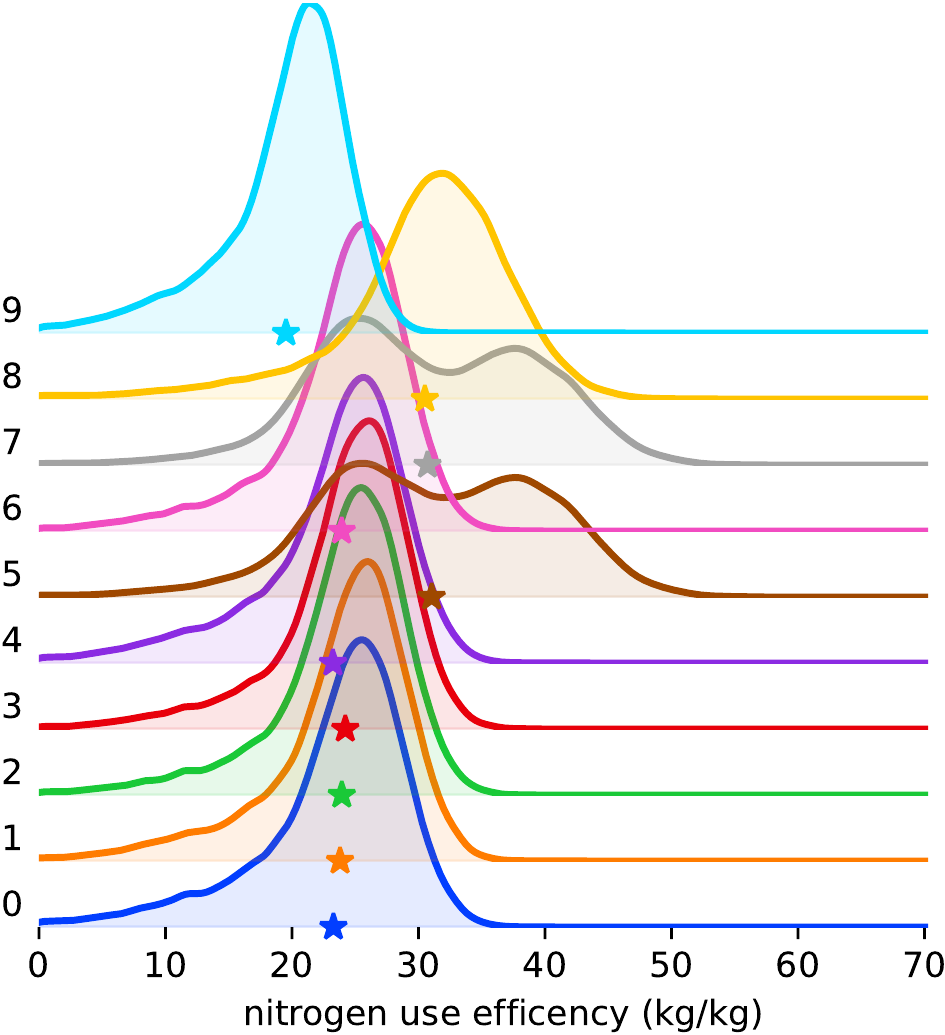}
		\caption{Agronomic Nitrogen Efficiency (ANE) distributions for soil \texttt{ITML840104}. Stars represent the mean value.}
		\label{fig:effDist04}
	\end{subfigure}%
	\hfill%
	\begin{subfigure}[t]{\widthFactorDists \textwidth}
		\centering
		\includegraphics[width=\textwidth]{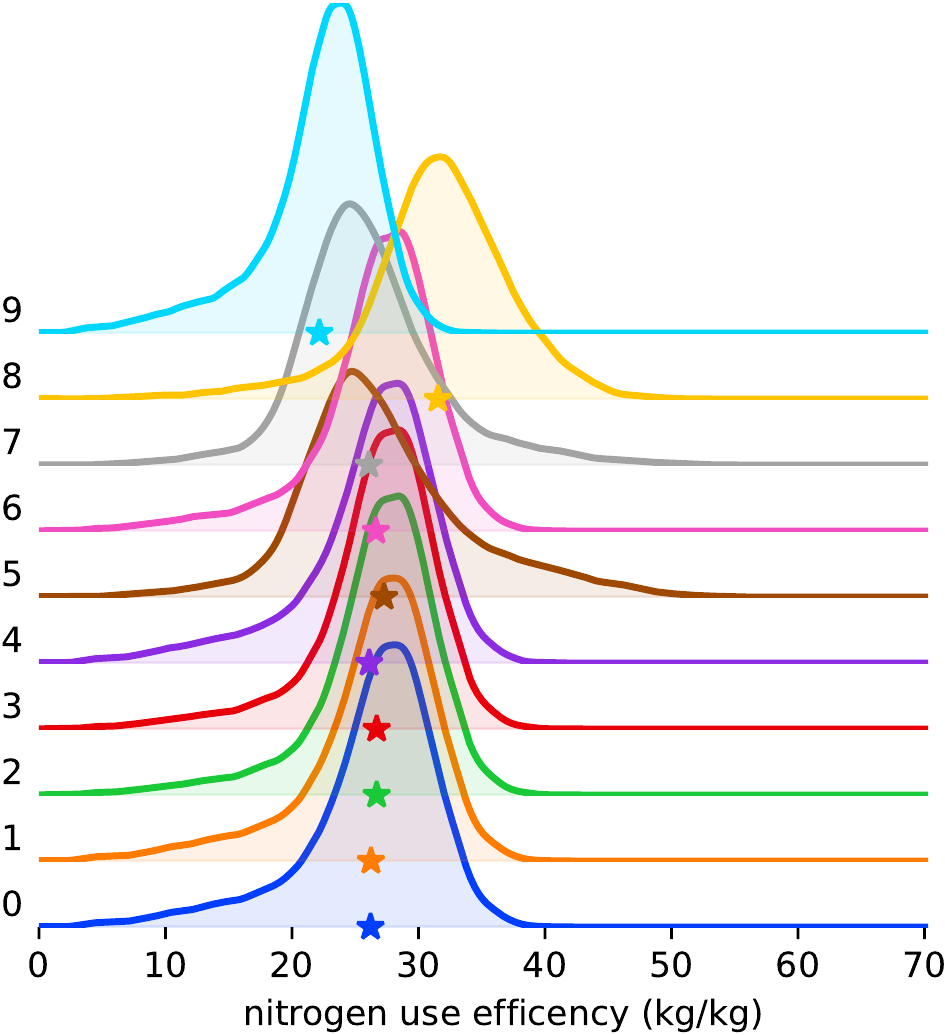}
		\caption{Agronomic Nitrogen Efficiency (ANE) distributions for soil \texttt{ITML840105}. Stars represent the mean value.}
		\label{fig:effDist05}
	\end{subfigure}
	\label{fig:yieldDists}
	\begin{subfigure}[t]{\widthFactorDists \textwidth}
    	\centering
    	\includegraphics[width=\textwidth]{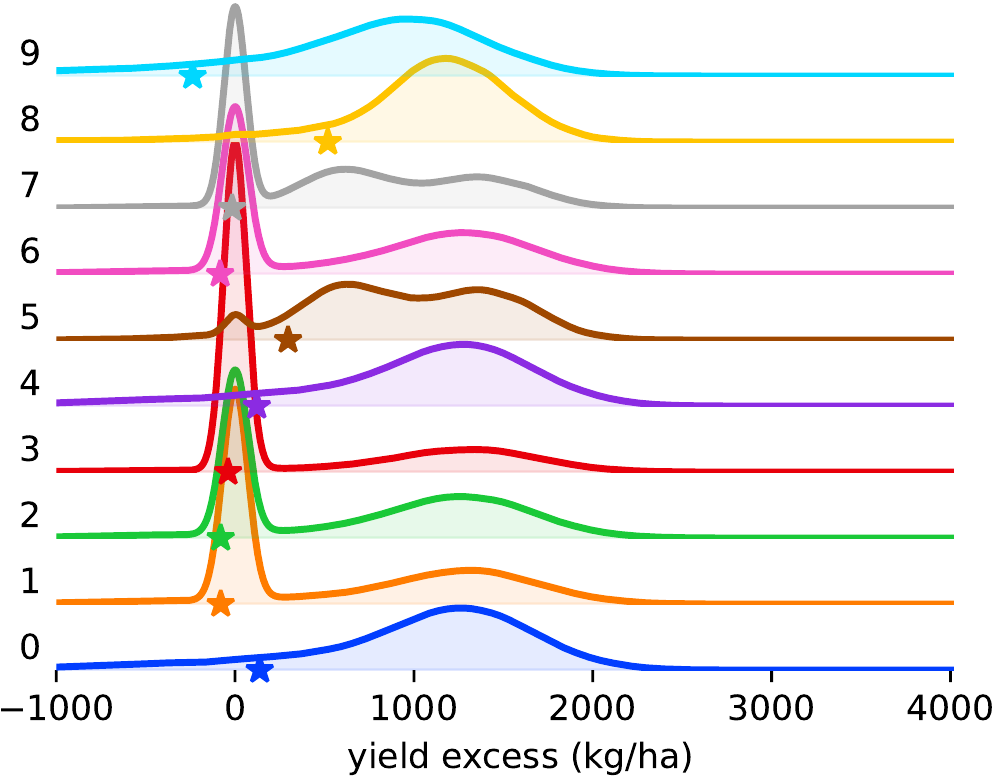}
    	\caption{Yield Excess (YE) distributions for soil \texttt{ITML840104} with  $\text{ANE}_{\text{ref}}$=15 kg grain/kg N. Stars represent the CVaR at level 30\%.}
    	\label{fig:YEDist04}
    \end{subfigure}%
	\hfill%
	\begin{subfigure}[t]{\widthFactorDists \textwidth}
		\centering
		\includegraphics[width=\textwidth]{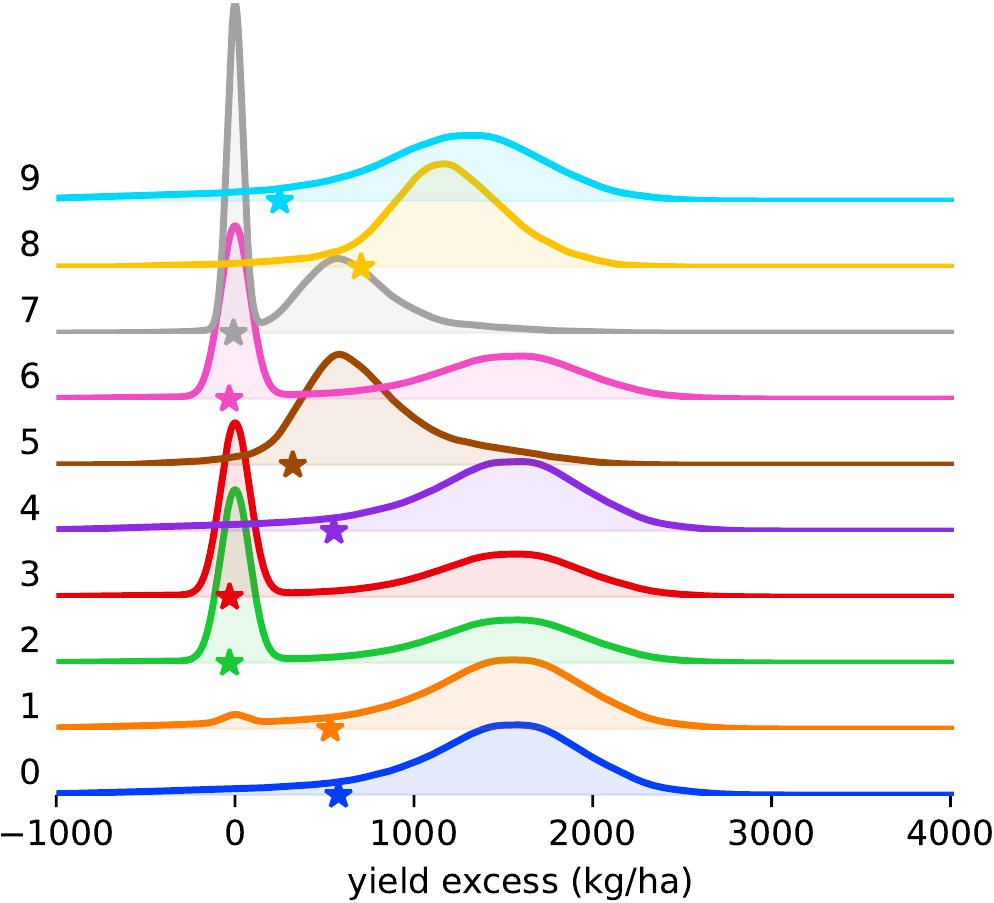}
		\caption{Yield Excess (YE) distributions for soil \texttt{ITML840105} with  $\text{ANE}_{\text{ref}}$=15 kg grain/kg N. Stars represent the CVaR at level 30\%.}
		\label{fig:YEDist05}
	\end{subfigure}
	\caption{Simulated impact of maize fertilizer practices on grain yield, Agronomic Nitrogen use Efficiency (ANE), Yield Excess (YE) for $10^5$ hypothetical years using a weather generator. Maize cultivar was the same for all simulations. Practices indexes are indicated on the left-hand side of each sub-figure. \label{fig:responses}}
\end{figure}

%% file: appendices/algorithms.tex
\subsection{Details about \texttt{BCB}}
\label{sec:bcvtsDetailed}
In algorithm~\ref{algo:cohortAlgDetailed}, we provide the detailed pseudo-code of \texttt{BCB} (\texttt{BCB}). As shown by Figure~\ref{fig:dirichletSampling}, the higher the number of collected rewards, the less the weights sampled from Dirichlet distributions exhibit variance. This variance directly relates to the noise introduced in the computation of the score of the different available actions.

\begin{algorithm}[H]
	\caption{\texttt{BCB}: identification strategy at cohort level (detailed) \label{algo:cohortAlgDetailed}}
	\SetKwInput{KwData}{Input} 
	\KwData{Level $\alpha$, horizon $T$, $K$ options, upper bounds $B_1, \dots, B_K$, $\cF^c$ the set of all farmers in the cohort}
	\SetKwInput{KwResult}{Init.}
	\KwResult{$\forall k \in \{1, ..., K \}$: $\cX_k=\{B_k\}$, $N_k=0$ ; $\cF_1^c =  \{f_1, \cdots, f_{n_1}\}$ ; $t = 1$ ; $\cA_1 = \{\emptyset\}$}
	\tcp{Beginning of first season}
	\For{$f \in \cF_1^c$}{
	Randomly assign a crop management option $a \in \K$ to the farmer $f$\\
	$\cA_1 = \cA_1 \cup \{a\}$
	}
	\tcp{End of first season}
	\For{$(a, f) \in (\cA_1, \cF_1^c)$}{
	Receive the result of the option $a$ from farmer $f$: $r_{f,a}$\\
	Update $\cX_{a} = \cX_{a} \cup \{r_{f, a}\}, N_{a}=N_{a}+1$
	}
	\For{$ t \in \{2,\dots, T\}$ }{
		\tcp{Beginning of season $t$}
	    Get $\cF_t^c =  \{f_1, \cdots, f_{n_t}\}$ \tcp*{the set of farmers of the same cohort to provide recommendations}
		\For{$k \in \K$}{
		Update the empirical CVaR of action $k$:  $\hat{c}_{k, t-1} =\widehat{C}_\alpha(\cX_k)$
		}
		\For{$f \in \cF_t^c$}{
		Update the empirical regret of farmer $f$:  $l_{f, t-1} = \widehat{R}^{\alpha}_f(t-1)$
		}
		$\cA_t = \{\emptyset\}$  \tcp*{the set of recommendations to provide to the farmers}
		\For{$f \in \cF_t^c$}{
			\For{$k \in \K$}{
    			Draw $\omega_k=\{w_1, \cdots, w_{N_k}\} \sim \cD_{N_k}$ \tcp*{Dirichlet of concentration parameter $\underbrace{(1, \cdots, 1)}_{N_k ~\text{times}}$}
    			Search $j$ the maximum index such that $\sum_{i=1}^{j} w_i \leq \alpha$ \\
    			Sort $\cX_k$ in increasing order \\
    			Compute $\tilde{c}_{k} = x_j - \frac{1}{\alpha}\sum_{i=1}^{N_k} w_i \max(x_j - x_i, 0)$ \tcp*{assign a score to action $k$}
    		}
		    $a=\aargmax_{k \in \K} \tilde{c}_{k}$\\
		    $\cA_t = \cA_t \cup \{a\}$
		}
		\parbox[t]{\dimexpr\linewidth-\algorithmicindent}{Sort the set of farmers $\cF_t^c$ according their increasing empirical regrets $l_{f, t-1}$\\
		Sort the set of actions $\cA_t$ according their increasing empirical CVaR $\hat{c}_{k, t-1}$\\
		\For{$(a, f) \in (\cA_t, \cF_t^c)$}{
		Assign action $a$ to farmer $f$ \tcp*{fair exploration}}
		\tcp{End of season $t$}
		\For{$(a, f) \in (\cA_t, \cF_t^c)$}{
		Receive result of action a from farmer $f$: $r_{f,a}$ \\
		Update $\cX_{a} = \cX_{a} \cup \{r_{f, a}\}, N_{a}=N_{a}+1$}
		}
	}	
\end{algorithm}

\begin{figure}
    \centering
    \includegraphics[width=.5\textwidth]{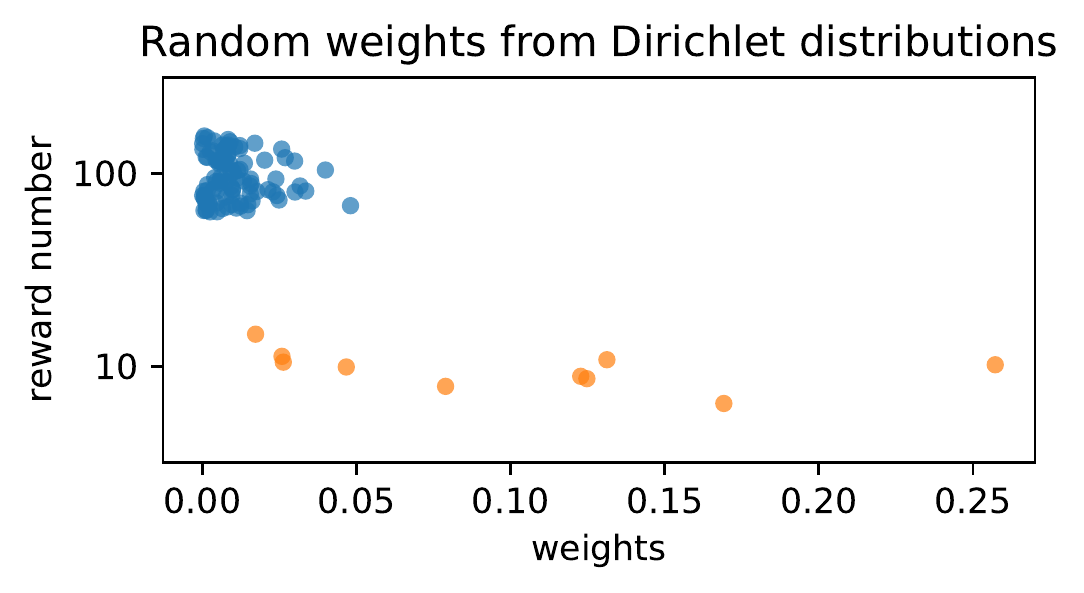}
    \caption{Examples of weights sampled from Dirichlet distributions during \texttt{BCB} execution, respectively for 10 and 100 rewards. The greater the number of rewards, the less variance the weights show. The variance of weights is related to the noise level in the computation of the empirical CVaR of \texttt{BCB}.}
    \label{fig:dirichletSampling}
\end{figure}

\begin{remark}[First season]\label{rem:firstSeason} Algorithm~\ref{algo:cohortAlgDetailed} is well defined for the first season as without data all CVaRs will be equal to the maximum observable result, making the algorithm choose each option arbitrarily at random. On average, each option will be equally explored. Note that we could replace this step by an equi-proportional exploration step (similar to Explore-Then-Commit, see~\ref{sec:ETC}) without changing the theoretical properties of our algorithm. Furthermore, the decision maker could also include any additional results collected before the experiment (if the practices has already been tested for some time) in the initialization of the algorithm.
\end{remark}

\subsection{Explore-Then-Commit (\texttt{ETC})}
\label{sec:ETC}

We provide the pseudo-code of the Explore-Then-Commit (\texttt{ETC}) strategy with algorithm \ref{algo:ETC}. The noise introduced by random weights and the presence of the maximum observable results in the histories manage the exploration/exploitation dilemma. \texttt{BCB} will favor fertilizer practices with higher CVaR compared to the others. But, the algorithm will still prevent the under-exploration of fertilizer practices by choosing them with a proper probability, even if e.g.\@ poor YE have been observed due to rare unfavorable weather events. Indeed, with the extra randomness introduced by the random weighting of rewards, poor rewards may be re-weighted by smaller weights compared to higher rewards, yielding a good score. The amount of noise introduced by the random weights sampled from the Dirichlet distribution is related to variance of these random weights. The greater the number of rewards, the lesser the variance and consequently the lesser the noise (Figure~\ref{fig:dirichletSampling}). Thereby, the more a fertilizer practice was tried by the algorithm, the closer its score gets to the true CVaR of rewards. The presence of the maximum observable YE acts as an ``optimistic bonus" in the computation of the scores, encouraging exploration even for sub-optimal practices, as it raises up their initial values when few rewards have been observed.

\begin{algorithm}[H]
	\caption{\texttt{ETC}: identification strategy at cohort level \label{algo:ETC}}
	\SetKwInput{KwData}{Input} 
	\KwData{Level $\alpha$, horizon $T$, $K$ options, $\cF^c$ the set of all farmers in the cohort, $t_{\text{trials}}$ the number of years of trials}
	\SetKwInput{KwResult}{Init.}
	\KwResult{$\forall k \in \{1, \cdots, K\}: N_k=0$}
	\tcp{Do trials during $t_{\text{trials}}$ years}
	\For{$t \in \{1, \cdots, t_{\text{trials}}\}$}{
    	\tcp{Beginning of the season $t$}
    	Get $\cF_t^c =  \{f_1, \cdots, f_{n_t}\}$ \tcp*{get the farmers willing to participate}
    	$\cA_t = \{\emptyset\}$ \\
    	Fill $\cA_t$ by uniformly distributing the $K$ options to the farmers in $\cF^c_t$ \\
    	\tcp{End of the season $t$}
    	\For{$(a,f) \in (\cA_t, \cF_t^c)$}{
    		Receive the result of the option $a$ from farmer $f$: $r_{f,a}$ \\
    	    Update $\cX_{a} = \cX_{a} \cup \{r_{f, a}\}, N_{a}=N_{a}+1$
	    }
	}
	\For{$k \in \K$}{
    	Compute the empirical CVaR of action $k$:  $\hat{c}_{k, t-1} =C_\alpha(\cX_k)$
    }	
    $a_{\max} = \aargmax_{k \in \K} \hat{c}_{k}$
    \tcp*{get the action that best performed during trials}
    \tcp{After trial phase, always recommend the action that best performed during trials}
	\For{$t \in \{t_{\text{trials}} + 1, \cdots, T\}$}{
	    \tcp{Beginning of the season $t$}
	    Get $\cF_t^c =  \{f_1, \cdots, f_{n_t}\}$ \\
    	\For{$f \in \cF_1^c$}{
    	Assign option $a_{\max}$ to the farmer $f$
    	}
        \tcp{End of the season $t$}
        \For{$f \in \cF_t^c$}{
    		Receive the result of the option $a_{\max}$ from farmer $f$: $r_{f,a_{\max}}$ \\
    	    Update $\cX_{a_{\max}} = \cX_{a_{\max}} \cup \{r_{f, a_{\max}}\}, N_{a_{\max}}=N_{a_{\max}}+1$
	    }
	}
\end{algorithm}

%% file: appendices/additionalExp.tex
Following methods of Section~\ref{sec:methodsBatchBandit} of the main text, we provide identification performances of identification strategies for CVaR levels $\alpha=50\%$ and $\alpha=100\%$ with Figures~\ref{fig:complementEmpiricalCVaR},~\ref{fig:complementGroupRegret} and ~\ref{fig:complementFarmerRegret}. For both CVaR levels, the YE is defined with $\text{ANE}_{\text{ref}}=15$ kg N/kg grain.
\newcommand{\widthFactorExtraExpp}{.48}
\begin{figure}
	\centering
	\begin{subfigure}[t]{\widthFactorExtraExpp \textwidth}
		\centering
		\includegraphics[width=\textwidth]{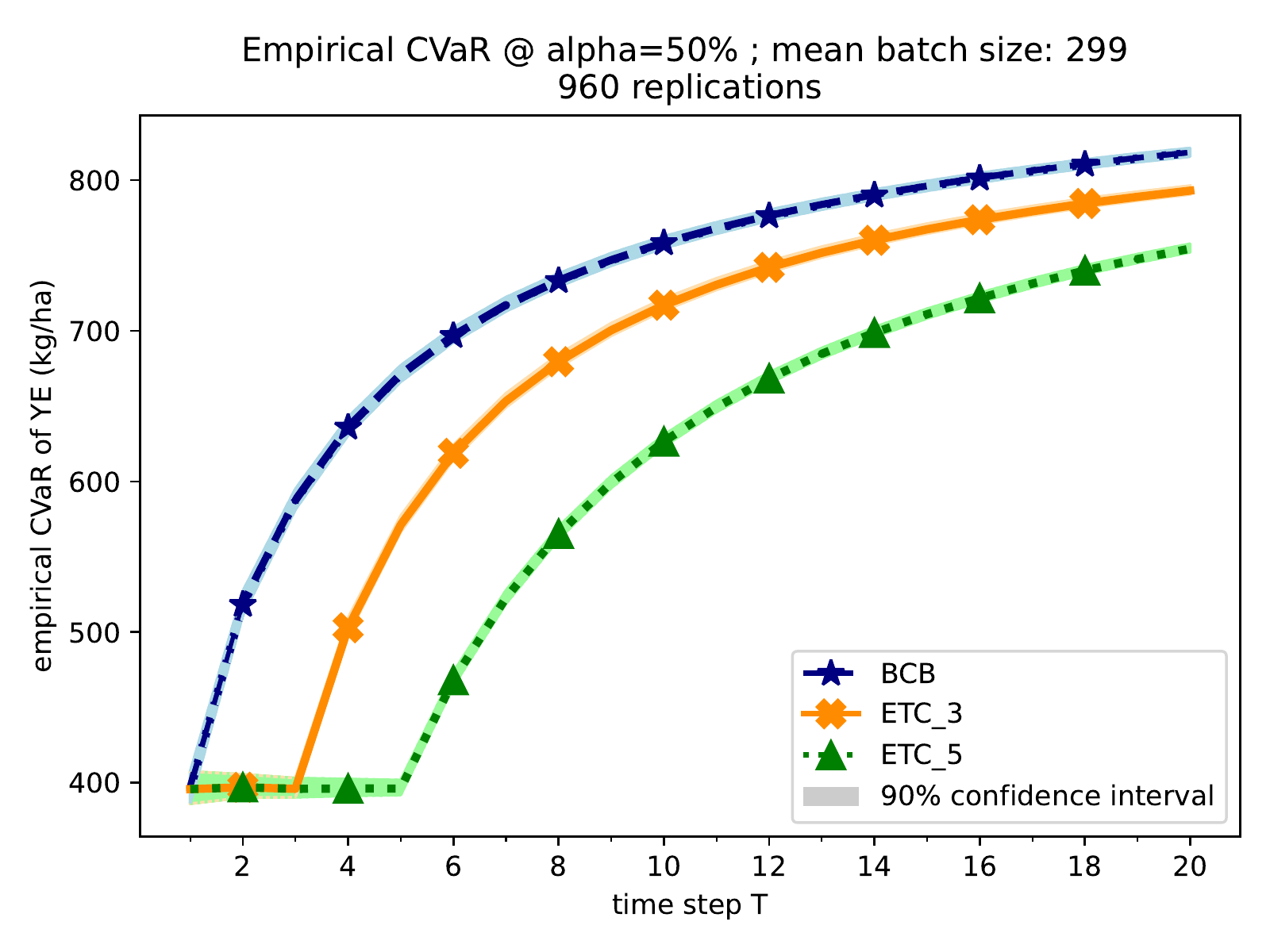}
		\caption{$\alpha=50\%$}
	\end{subfigure}%
	\hfill%
	\begin{subfigure}[t]{\widthFactorExtraExpp \textwidth}
		\centering
		\includegraphics[width=\textwidth]{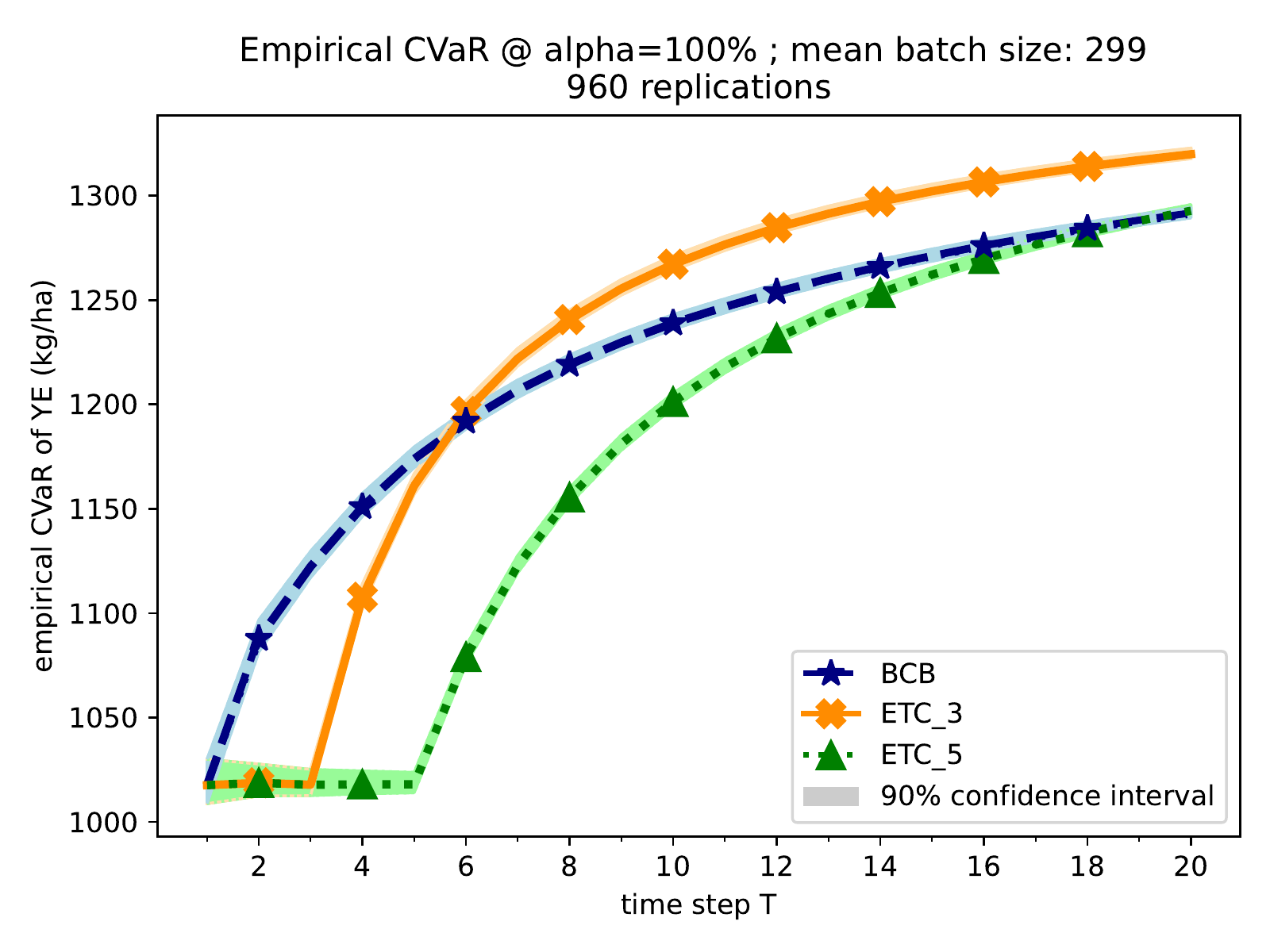}
		\caption{$\alpha=100\%$}
	\end{subfigure}
\caption{Farmers' empirical CVaR at level of all YE received between $T=0$ and the considered $T$. \label{fig:complementEmpiricalCVaR}}
\end{figure}

\begin{figure}
    \begin{subfigure}[t]{\widthFactorExtraExpp \textwidth}
    	\centering
    	\includegraphics[width=\textwidth]{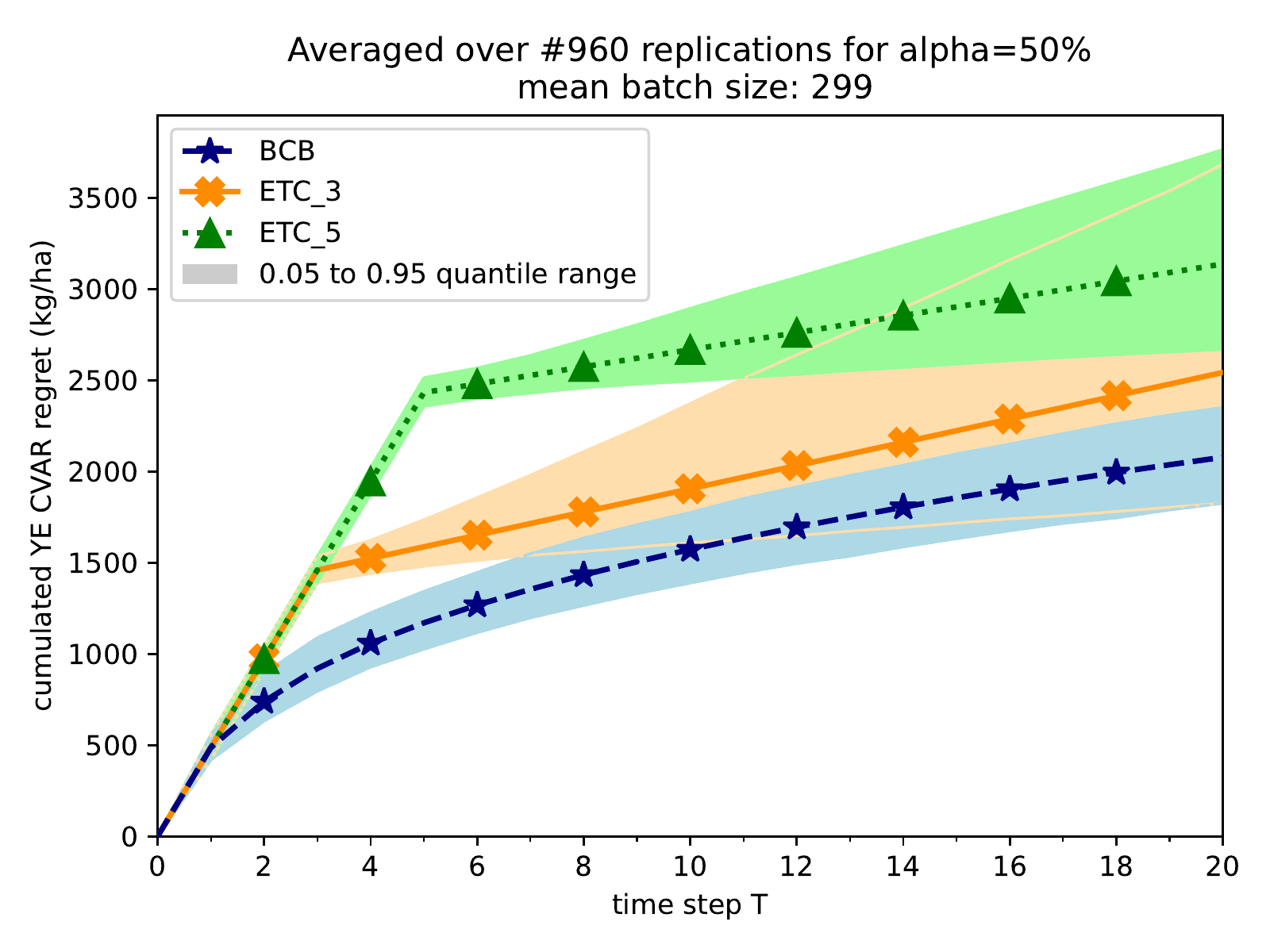}
    	\caption{$\alpha=50\%$}
    \end{subfigure}%
    \hfill%
    \begin{subfigure}[t]{\widthFactorExtraExpp \textwidth}
    	\centering
    	\includegraphics[width=\textwidth]{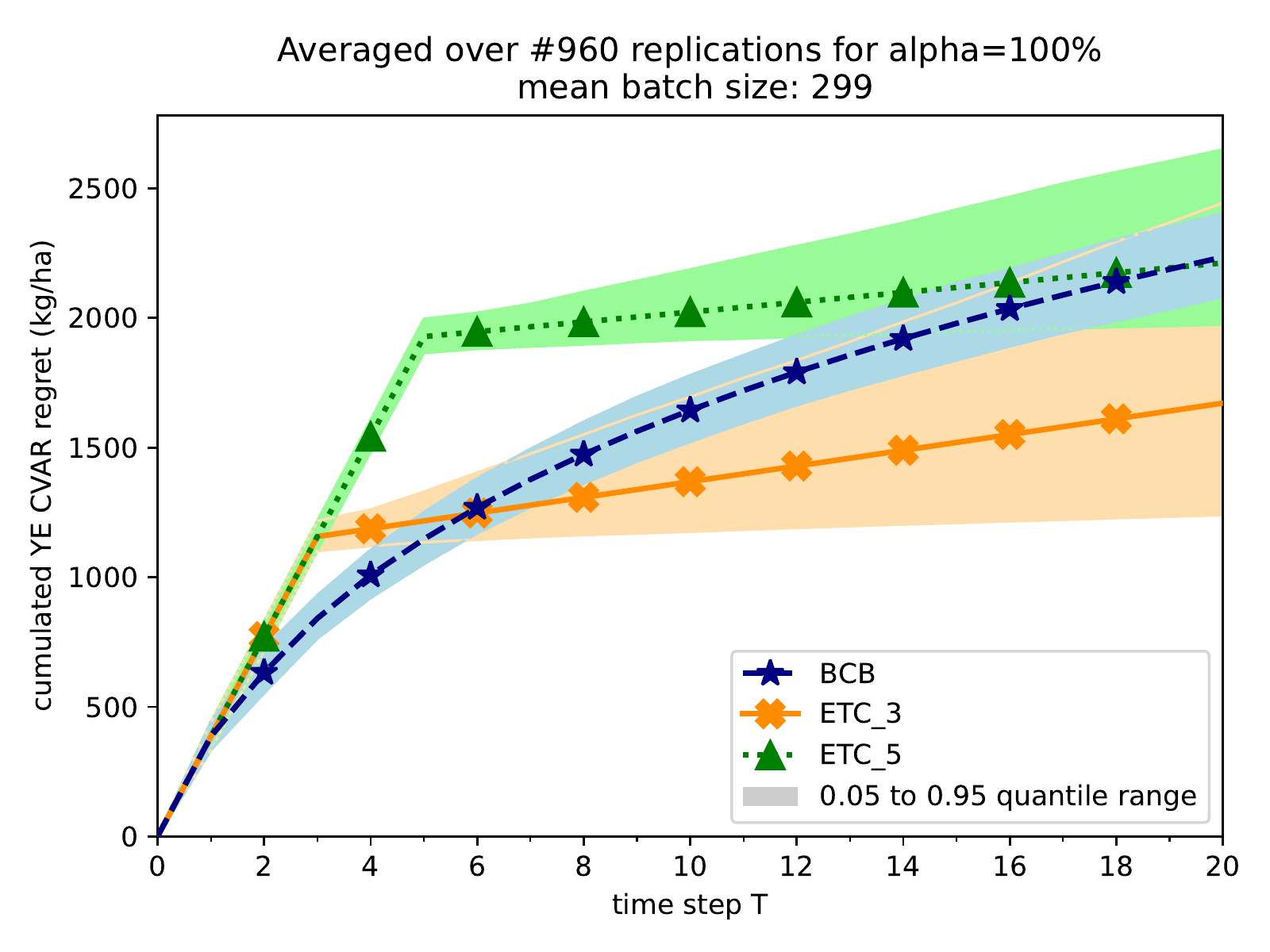}
    	\caption{$\alpha=100\%$}
    \end{subfigure}
\caption{Cumulated regret averaged over the population for the CVaR at level of YE.     \label{fig:complementGroupRegret}}
\end{figure}
	
\begin{figure}
	\begin{subfigure}[t]{\widthFactorExtraExpp \textwidth}
    	\centering
    	\includegraphics[width=\textwidth]{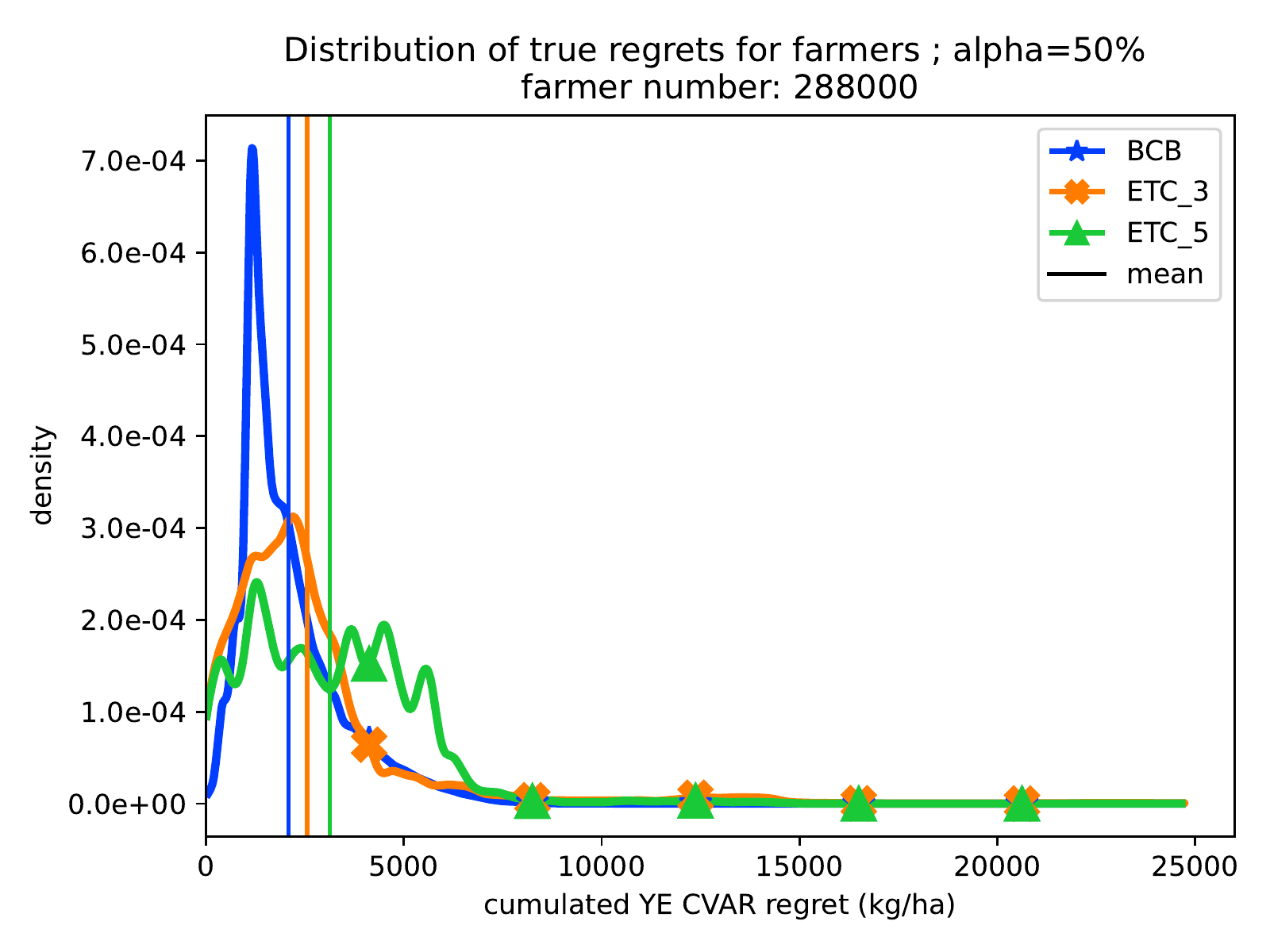}
		\caption{$\alpha=50\%$}
    \end{subfigure}%
    \hfill%
    \begin{subfigure}[t]{\widthFactorExtraExpp \textwidth}
    	\centering
    	\includegraphics[width=\textwidth]{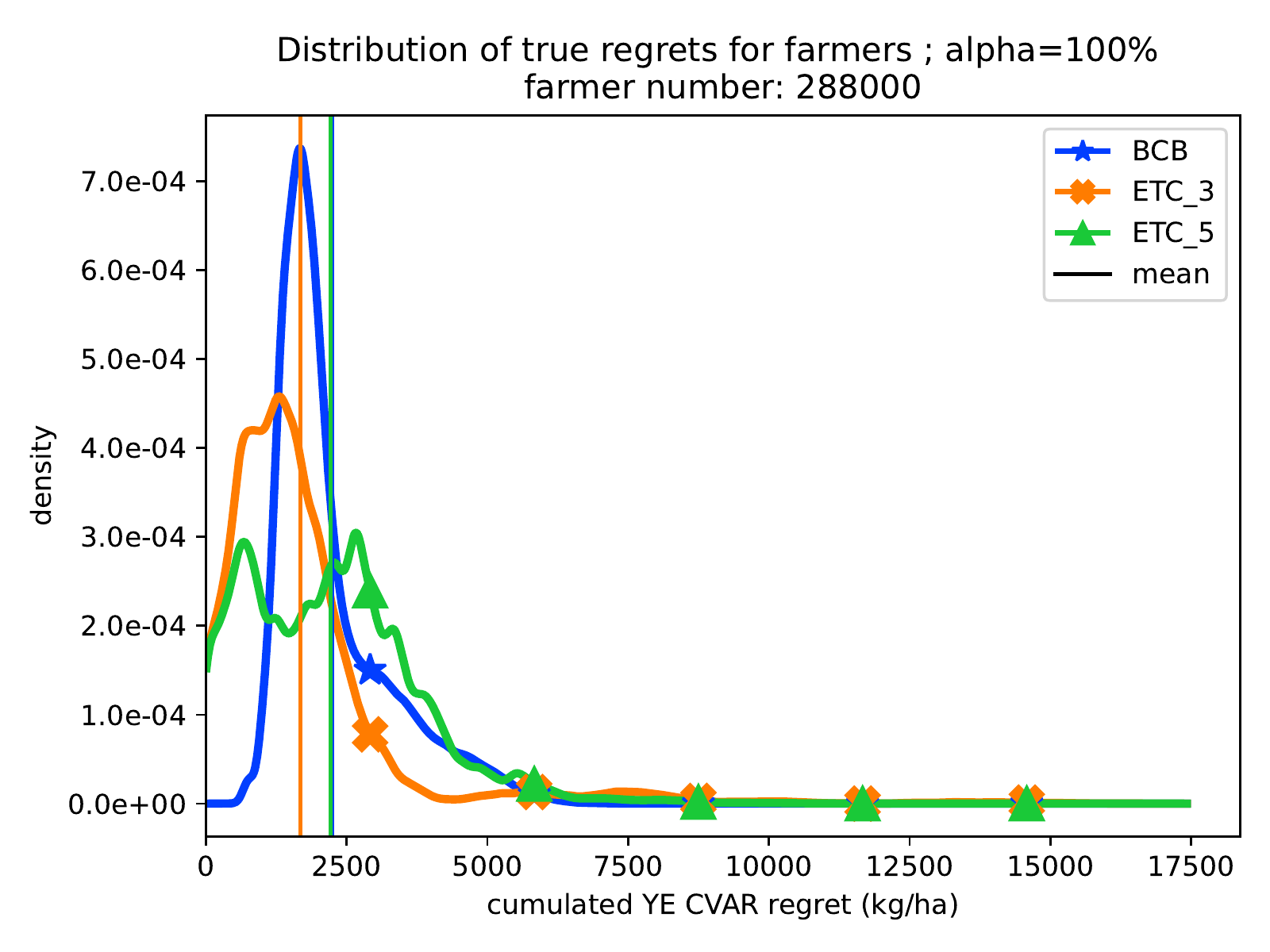}
		\caption{$\alpha=100\%$}
    \end{subfigure}
\caption{Distribution of individual cumulated regret after $T=20$. \label{fig:complementFarmerRegret}}
\end{figure}

%% file: appendices/regretBatch.tex
This section is devoted to the theoretical analysis of the \texttt{BCB} algorithm. We will mostly adapt the analysis of \cite{baudry2021optimal}, and show that the problem of learning with batched data of finite upper bounded size is no harder than the pure online learning problem considered in the original paper. 

\begin{theorem}[$\alpha$-CVaR Regret of BCB]\label{th:regret_bcb}
Consider a bandit problem $(F_1,\dots, F_K) \in \cF^K$, with respective CVaR$_\alpha$ denoted by $(c_1,\dots, c_K)$ with $c_1=\text{argmax}_{k = 1,\dots, K} c_k$. Assume that BCB runs for $T$ seasons, and that at each season the size of the batch is $n_T\leq F \in \N$. Then, for any $\epsilon>0$ small enough there exists some $\epsilon_1>0, \epsilon_2>0$ such that the regret of BCB satisfies 
\[\cR_T^\alpha \leq \sum_{k=2}^K \Delta_k^\alpha \left( m_T^k + F + 2F \frac{e^{-2m_T^k \epsilon_1^2}}{1-e^{-2\epsilon_1^1}} + C_{1,\epsilon_2}^\alpha\right)\;, \]
where $m_T^k = \frac{\log(T)+\log(F)}{\kinfcv(F_k, c_1)-\epsilon}$ and $C_{1,\epsilon_2}$ is a constant depending only on the distribution $F_1$, the family $\cF$ and $\epsilon_2$.
\end{theorem}

It is interesting to compare this regret upper bound to the one obtained in the purely sequential setting, that we recall in Theorem~\ref{th:regret_bcvt}. 
\begin{theorem}[$\alpha$-CVaR Regret of B-CVTS with time horizon $S_T$ (adapted from Theorem 3 in \cite{baudry2021optimal})]
Consider a bandit problem $(F_1,\dots, F_K) \in \cF^K$, with respective CVaR$_\alpha$ denoted by $(c_1,\dots, c_K)$ with $c_1=\text{argmax}_K c_k$. Consider a number of data collected $S_T$. Then, for any $\epsilon>0$ small enough there exists some $\epsilon_1>0, \epsilon_2>0$ such that the CVaR-regret of B-CVTS satisfies 
\[\cR_T^\alpha \leq \sum_{k=2}^K \Delta_k^\alpha\left(n_{S_T}^k + 2 \frac{e^{-2n_{S_T}^k \epsilon_1^2}}{1-e^{-2\epsilon_1^1}} + C_{1,\epsilon_2}^\alpha\right)\;, \]
where $m_{S_T}^k = \frac{\log(S_T)}{\kinfcv(F_k, c_1)-\epsilon}$ and $C_{1,\epsilon_2}$ is a constant depending only on the distribution $F_1$, the family $\cF$ and $\epsilon_2$.
\label{th:regret_bcvt}
\end{theorem}

First, we see that if $F$ is indeed a constant (i.e do not depend on the time) then when $T$ is large enough then $F$ has not impact on the scaling of the regret. In our proof the main impact of the batch setting is an additive term $F$ for each arm, hence the regret becomes close to the one of the sequential setting once $m_T^k \gg F$. Finally, if the number of farmers in each batch is exactly $F$ at each step then $S_T=FT$ and, $m_T^k=n_{S_T}^K$, hence the asymptotically dominant (logarithmic) term is the same in the two settings.

These theoretical results show that learning with batch feedback does not introduce theoretical limitations in our setting, and so the BCB algorithm is theoretically grounded.

\begin{proof}[Proof of Theorem \ref{th:regret_bcb}.]

As in the proof of \cite{baudry2021optimal} we will decompose the expected number of pulls of each sub-optimal arm inside the cohort according to several possible events, corresponding to "good" scenarios (the empirical distributions accurately reflect the true distributions) and "bad" ones (the empirical distributions give a wrong idea of the true performance of some arms) for the trajectory of the bandit algorithms. We denote by $T$ the number of seasons in the experiments and $n_t$ the number of farmers at each season $t$ for this cohort, and by $F$ the total number of farmers available for the experiment. Then, the expected number of pulls of arm $k$ during the total duration of the experiment inside the cohort is \[\bE[N_k(T)]=\bE\left[\sum_{t=1}^T \sum_{f=1}^{n_t} \ind(A_{t, f} = k) \right] \;,\]
where $A_{t, f}$ denotes the recommendation to farmer $f$ at season $t$.

The first step of the proof consists in considering the number of pulls of $k$ when its sample size is larger (resp. smaller) than some fixed threshold $m_T$, that we will specify later.

\begin{align*}
    \bE[N_k(T)]&=\bE\left[\sum_{t=1}^T \sum_{f=1}^{n_t} \ind(A_{t, f} = k) \right] \\
    & \leq \bE\left[\sum_{t=1}^T \sum_{f=1}^{n_t} \ind(A_{t, f} = k, N_k(t-1) \leq m_T) \right] +\bE\left[\sum_{t=1}^T \sum_{f=1}^{n_T} \ind(A_{t, f} = k, N_k(t-1) \geq m_T) \right]\\
\end{align*}

We now consider the first term and introduce the random variable $\tau = \{\sup_{t \leq T}: N_k(t-1)\leq m_T\}$. By construction, $\tau$ is the last season for which the total number of observations for arm $k$ is smaller than $m_T$. Using the basic properties of $\tau$ we obtain that

\begin{align*}
  \sum_{t=1}^T \sum_{f=1}^{n_t} \ind(A_{t, f} = k, N_k(t-1) \leq m_T)  & \leq \sum_{t=1}^{\tau} \sum_{f=1}^{n_t} \ind(A_{t, f} = k, N_k(t-1) \leq m_T) + \sum_{t=\tau+1}^{T} \sum_{f=1}^{n_t} \ind(A_{t, f} = k, N_k(t-1) \leq m_T)\\
  & \leq  N_k(\tau) + \sum_{f=1}^{n_{\tau+1}} \ind(A_{\tau, f} = k) \\
  & \leq m_T + F \\
\end{align*}

As this result does not depend on the value of $\tau$, we can then obtain 

\[ \bE[N_k(T)] \leq m_T + F + \underbrace{ \bE\left[\sum_{t=1}^T \sum_{f=1}^{n_t} \ind(A_{t, f} = k, N_k(t-1) \geq m_T) \right]}_{A} \;.\]

At this step, the only difference with the purely sequential bandit problem is the additional $F$. We now consider the term $A$, that we further analyze according to three events: (1) the empirical distribution of arm $k$ is not close to its true distribution, (2) the empirical distribution of arm $k$ is close to its true distribution but the "noisy" CVaR computed for arm $k$ over-estimates its true CVaR, and (3) the "noisy" CVaR computed for the optimal arm $1$ under-estimates its true CVaR. Classically in bandit analysis, we decompose the number of pulls of arm $k$ according to these three events, as at least one of them must be true when $A_{t,f}=k$ holds, that is
\[\{A_t=k\} \subset \{F_{k, t-1} \notin \cB_{\epsilon_1}(F_k) \} \cup \{F_{k, t-1} \in \cB_{\epsilon_1}(F_k), \widetilde c_{k, t, f} \geq c_1-\epsilon_2\} \cup \{\widetilde c_{1, t, f} \leq c_1 - \epsilon_2\} \;,\]

where $\cB_{\epsilon_1}(F_k)$ is an $\epsilon_1$-Levy ball around $F_k$, and $\epsilon_1, \epsilon_2$ are two small positive constants. This leads to 

\begin{align*}A &\leq \underbrace{\bE\left[\sum_{t=1}^T \sum_{f=1}^{n_t} \ind(A_{t, f} = k, N_k(t-1) \geq m_T, F_{k, t-1} \notin \cB_{\epsilon_1}(F_k)) \right]}_{A_1} \\
&+ \underbrace{\bE\left[\sum_{t=1}^T \sum_{f=1}^{n_t} \ind(A_{t, f} = k, N_k(t-1) \geq m_T, F_{k, t-1} \in \cB_{\epsilon_1}(F_k), \widetilde c_{k, t, f} \geq c_1-\epsilon_2) \right]}_{A_2} \\
&+ \underbrace{\bE\left[\sum_{t=1}^T \sum_{f=1}^{n_t} \ind(A_{t, f} = k, N_k(t-1) \geq m_T, \widetilde c_{1, t, f} \leq c_1 - \epsilon_2) \right]}_{A_3} \;. \end{align*}

\paragraph{Upper bounding $A_2$} Denoting by $\widehat F_{k, n}$ the empirical distribution of arm $k$ after a total number of \textit{pulls} $n$ (instead of after season $t$), we obtain 

\begin{align*}
    A_1 & \coloneqq \bE\left[\sum_{t=1}^T \sum_{f=1}^{n_t} \ind(A_{t, f} = k, N_k(t-1) \geq m_T, F_{k, t-1} \notin \cB_{\epsilon_1}(F_k)) \right] \\
    & \leq \bE\left[\sum_{t=1}^T  \ind(N_k(t-1) \geq m_T, F_{k, t-1} \notin \cB_{\epsilon_1}(F_k)) \sum_{f=1}^{n_t} \ind(A_{t, f} = k) \right] \\
    & \leq  \bE\left[\sum_{t=1}^T \sum_{n=m_T}^{T} \ind(N_k(t-1) =n, F_{k, t-1} \notin \cB_{\epsilon_1}(F_k))\sum_{f=1}^{n_t} \ind(A_{t, f} = k)\right] \;, \\
\end{align*}
with a union bound on the number of pulls. Under $N_k(t-1)=n$ it holds that $F_{k, t-1}=\widehat F_{k,n}$, and so we can further write that

\begin{align*}
A_1    & \leq  \bE\left[\sum_{t=1}^T \sum_{n=m_T}^{T} \ind(N_k(t-1)=n, \widehat F_{k, n} \notin \cB_{\epsilon_1}(F_k))\sum_{f=1}^{n_t} \ind(A_{t, f} = k)\right] \\
    & \leq  \bE\left[\sum_{n=m_T}^{T} \ind(\widehat F_{k, n} \notin \cB_{\epsilon_1}(F_k)) \sum_{t=1}^T \sum_{f=1}^{n_t} \ind(A_{t, f} = k, N_k(t-1)=n)\right] \\
    & \leq F \bE\left[\sum_{n=m_T}^{T} \ind(\widehat F_{k, n} \notin \cB_{\epsilon_1}(F_k))\right] \\
    & = F \sum_{n=m_T}^{+\infty} \bP(F_{k, n} \notin \cB_{\epsilon_1}(F_k)) \\
\end{align*}

Finally, using the Dvoretzky–Kiefer–Wolfowitz inequality \citep{massart1990} we obtain
\begin{align*}
    & \leq F \sum_{n=m_T}^{+\infty} 2 e^{-2n \epsilon_1^2} \\
    & \leq \frac{2F e^{-2m_T \epsilon_1^2}}{1-e^{-2\epsilon_1^2}} \;. \\
\end{align*}

This upper bound holds for any choice of $m_T, \epsilon_1$, and we remark that if $m_T \rightarrow +\infty$ then $A_1 \rightarrow 0$.

\paragraph{Upper bounding $A_2$}The term $A_2$ is then handled with similar tricks, and the arguments used in \cite{baudry2021optimal}.

\begin{align*}
    A_2 & \coloneqq  \bE\left[\sum_{t=1}^T \sum_{f=1}^{n_t} \ind(A_{t, f} = k, N_k(t-1) \geq m_T, F_{k, t-1} \in \cB_{\epsilon_1}(F_k), \widetilde c_{k, t, f} \geq c_1-\epsilon_2) \right]\\
        & \leq \bE\left[\sum_{t=1}^T \sum_{f=1}^{F} \ind(N_k(t-1) \geq m_T, F_{k, t-1} \in \cB_{\epsilon_1}(F_k)) \times \bP\left(\widetilde c_{k, t, f} \geq c_1-\epsilon_2 | \cF_t\right) \right] \;, \\
\end{align*}

where $\cF_t$ is the canonical filtration, so the probability is obtained conditioning on the data observed before the beginning of the round. Using the the continuity of $\kinfcv$ in its two arguments as proved in \cite{agrawal2020optimal}, we obtain that for any $\epsilon>0$ small enough there exist some $\epsilon_1, \epsilon_2$ such that 
\begin{align*}
    A_2 & \leq \bE\left[\sum_{t=1}^T \sum_{f=1}^{F} \ind(A_{t, f} = k, N_k(t-1) =n, F_{k, t-1} \in \cB_{\epsilon_1}(F_k)) e^{-m_T \left( \kinfcv(F_k, c_1)-\epsilon\right)} \right] \\
    & \leq F \times T \times e^{-m_T \left( \kinfcv(F_k, c_1)-\epsilon\right)} \;.\\
\end{align*}

As we did not specify the choice of $\epsilon_1, \epsilon_2$ already we simply require them to be small enough to satisfy this condition. Then, we can calibrate $m_T$ as \[m_T = \frac{\log(T)+\log(F)}{\kinfcv(F_k, c_1)-\epsilon}\;,\]

Furthermore, with this choice $m_T$ will become the main term in the regret upper bound when $T$ becomes large enough.

\paragraph{Upper bounding $A_3$} The final term is the one that leading to the most complicated part of the analysis in \cite{baudry2021optimal}. Fortunately, the batch setting will have no impact on this part, so we can directly reuse the results provided in this paper.

Indeed, we can re-write $A_3$ to make it equivalent to the corresponding term in the purely sequential problem: \[A_3 = \bE\left[\sum_{t=1}^T \sum_{f=1}^{n_t} \ind(\widetilde c_{1, t, f} \leq c_1 - \epsilon_2)\right] = \bE\left[\sum_{r=1}^{S_T} \ind(\widetilde c_{1}(r) \leq c_1 - \epsilon_2)\right] \;, \]

where in the second term we count the number of recommendations provided by the algorithm, assigning those in the same batch an arbitrary order, $\widetilde c_1(r)$ is then the noisy CVaR computed for arm $1$ for this specific round. Furthermore, we write $S_T = \sum_{t=1}^T n_t \leq FT$. In \cite{baudry2021optimal}, the authors obtain a constant upper bound for this term, depending only on $\epsilon_2$ (and the upper bound of the support), and in particular not depending on the exact number of plays. We conclude that there exists some constant $C_{1, \epsilon_2}$ satisfying \[A_3 \leq C_{1, \epsilon_2} \;.\] 

This result concludes our proof, and we refer the interested reader to the original paper for a complete proof and a detailed expression for $C_{1, \epsilon_2}$. We further remark that contrarily to the previous terms, the upper bound of $A_3$ does not depend on $F$ at all.
\end{proof}

%% file: appendices/YEcomplement.tex
\paragraph{}We briefly discuss economical criteria we considered as performance indicators of fertilizer practices. A first indicator we considered was the gross margin. The cost of production of nitrogen fertilizer being indexed on the price of natural gas, it is subject to high volatility. As a consequence, an optimal practice is likely to be different each year and thus the decision problem would turn to be highly non-stationary. Such setting dramatically increases the complexity of the decision problem, and the chance of observing good identification performances are lowered.

\paragraph{}Another economic measure could be the value:cost ratio (VCR), which is given for a fertilizer practice $\pi$ as:
\begin{align}
\text{VCR}^{\pi} &= \frac{{p}_{\text{maize}}}{{p}_{\text{N}}} \times \frac{Y^{\pi} - Y^0}{\text{N}^{\pi}}\\
&= \frac{{p}_{\text{maize}}}{{p}_{\text{N}}} \times \text{ANE}^{\pi}
\end{align}
where ${p}_{\text{N}}$ is fertilizer unitary cost and ${p}_{\text{maize}}$ unitary maize grain selling price. Remarking that each given year the ratio $\frac{{p}_{\text{maize}}}{{p}_{\text{N}}}$ is shared by all fertilizer practices. We neglect a possible quality consideration that could motivate a different maize selling price between the fertilizer practices, for instance a difference of protein content in maize grains. Then the decision problem is perfectly equivalent to choosing the fertilizer practice which maximizes the ANE. Thereby, the use of the cost:value ratio suffers from the same drawbacks as the ANE.

%% file: main.bbl
\begin{thebibliography}{}

\bibitem[Acerbi and Tasche, 2002]{acerbi2001cvar}
Acerbi, C. and Tasche, D. (2002).
\newblock On the coherence of expected shortfall.
\newblock {\em Journal of Banking \& Finance}, 26:1487--1503.

\bibitem[Adam et~al., 2020]{adam2020more}
Adam, M., MacCarthy, D.~S., Traor{\'e}, P. C.~S., Nenkam, A., Freduah, B.~S.,
  Ly, M., and Adiku, S.~G. (2020).
\newblock Which is more important to sorghum production systems in the
  sudano-sahelian zone of west africa: Climate change or improved management
  practices?
\newblock {\em Agricultural Systems}, 185:102920.

\bibitem[Affholder, 1995]{affholder1995effect}
Affholder, F. (1995).
\newblock Effect of organic matter input on the water balance and yield of
  millet under tropical dryland condition.
\newblock {\em Field Crops Research}, 41(2):109--121.

\bibitem[Affholder et~al., 2013]{affholder2013yield}
Affholder, F., Poeydebat, C., Corbeels, M., Scopel, E., and Tittonell, P.
  (2013).
\newblock The yield gap of major food crops in family agriculture in the
  tropics: Assessment and analysis through field surveys and modelling.
\newblock {\em Field Crops Research}, 143:106--118.

\bibitem[Agrawal et~al., 2021]{agrawal2020optimal}
Agrawal, S., Koolen, W.~M., and Juneja, S. (2021).
\newblock Optimal best-arm identification methods for tail-risk measures.
\newblock In {\em Advances in Neural Information Processing Systems 34: Annual
  Conference on Neural Information Processing Systems 2021, NeurIPS 2021,
  December 6-14, 2021, virtual}.

\bibitem[Baudron et~al., 2012]{baudron2012comparative}
Baudron, F., Tittonell, P., Corbeels, M., Letourmy, P., and Giller, K.~E.
  (2012).
\newblock Comparative performance of conservation agriculture and current
  smallholder farming practices in semi-arid zimbabwe.
\newblock {\em Field crops research}, 132:117--128.

\bibitem[Baudry et~al., 2021a]{baudry2021optimal}
Baudry, D., Gautron, R., Kaufmann, E., and Maillard, O. (2021a).
\newblock Optimal thompson sampling strategies for support-aware cvar bandits.
\newblock In {\em International Conference on Machine Learning}, pages
  716--726. PMLR.

\bibitem[Baudry et~al., 2021b]{baudry2021limited}
Baudry, D., Russac, Y., and Capp{\'e}, O. (2021b).
\newblock {On Limited-Memory Subsampling Strategies for Bandits}.
\newblock In {\em {ICML 2021- International Conference on Machine Learning}},
  Vienna / Virtual, Austria.

\bibitem[Cassel et~al., 2018]{cassel2018general}
Cassel, A., Mannor, S., and Zeevi, A. (2018).
\newblock A general approach to multi-armed bandits under risk criteria.
\newblock In {\em Conference On Learning Theory}, pages 1295--1306. PMLR.

\bibitem[Cerf and Meynard, 2006]{cerf2006outils}
Cerf, M. and Meynard, J.-M. (2006).
\newblock Les outils de pilotage des cultures: diversit{\'e} de leurs usages et
  enseignements pour leur conception.
\newblock {\em Natures Sciences Soci{\'e}t{\'e}s}, 14(1):19--29.

\bibitem[Cerf and Sebillotte, 1997]{cerf1997approche}
Cerf, M. and Sebillotte, M. (1997).
\newblock Approche cognitive des d{\'e}cisions de production dans
  l'exploitation agricole [confrontation aux th{\'e}ories de la d{\'e}cision].
\newblock {\em Economie rurale}, 239(1):11--18.

\bibitem[Dowd, 2007]{dowd2007measuring}
Dowd, K. (2007).
\newblock {\em Measuring market risk}.
\newblock John Wiley \& Sons.

\bibitem[Evans et~al., 2017]{evans2017data}
Evans, K.~J., Terhorst, A., and Kang, B.~H. (2017).
\newblock From data to decisions: helping crop producers build their actionable
  knowledge.
\newblock {\em Critical reviews in plant sciences}, 36(2):71--88.

\bibitem[Everitt and Skrondal, 2002]{everitt2002cambridge}
Everitt, B. and Skrondal, A. (2002).
\newblock {\em The Cambridge dictionary of statistics}, volume 106.
\newblock Cambridge University Press Cambridge.

\bibitem[Falconnier et~al., 2016]{falconnier2016unravelling}
Falconnier, G.~N., Descheemaeker, K., Van~Mourik, T.~A., and Giller, K.~E.
  (2016).
\newblock Unravelling the causes of variability in crop yields and treatment
  responses for better tailoring of options for sustainable intensification in
  southern mali.
\newblock {\em Field Crops Research}, 187:113--126.

\bibitem[Fosu-Mensah et~al., 2012]{fosu2012simulating}
Fosu-Mensah, B., MacCarthy, D., Vlek, P., and Safo, E. (2012).
\newblock Simulating impact of seasonal climatic variation on the response of
  maize (zea mays l.) to inorganic fertilizer in sub-humid ghana.
\newblock {\em Nutrient cycling in agroecosystems}, 94(2):255--271.

\bibitem[Garivier and Moulines, 2011]{garivier2011upper}
Garivier, A. and Moulines, E. (2011).
\newblock On upper-confidence bound policies for switching bandit problems.
\newblock In {\em International Conference on Algorithmic Learning Theory},
  pages 174--188. Springer.

\bibitem[Gautron et~al., 2022]{gautron2022reinforcement}
Gautron, R., Maillard, O.-A., Preux, P., Corbeels, M., and Sabbadin, R. (2022).
\newblock Reinforcement learning for crop management support: Review, prospects
  and challenges.
\newblock {\em Computers and Electronics in Agriculture}, 200:107182.

\bibitem[Gautron and {Padr{\'o}n Gonz{\'a}lez}, 2022]{gymdssat}
Gautron, R. and {Padr{\'o}n Gonz{\'a}lez}, E.~J. (2022).
\newblock {gym-DSSAT - A crop model turned into a Reinforcement Learning
  environment}.

\bibitem[Getnet et~al., 2016]{getnet2016yield}
Getnet, M., Van~Ittersum, M., Hengsdijk, H., and Descheemaeker, K. (2016).
\newblock Yield gaps and resource use across farming zones in the central rift
  valley of ethiopia.
\newblock {\em Experimental Agriculture}, 52(4):493--517.

\bibitem[Hanway, 1963]{hanway1963growth}
Hanway, J. (1963).
\newblock Growth stages of corn (zea mays, l.) 1.
\newblock {\em Agronomy Journal}, 55(5):487--492.

\bibitem[Hochman and Carberry, 2011]{hochman2011emerging}
Hochman, Z. and Carberry, P. (2011).
\newblock Emerging consensus on desirable characteristics of tools to support
  farmers’ management of climate risk in australia.
\newblock {\em Agricultural Systems}, 104(6):441--450.

\bibitem[Hoogenboom et~al., 2019]{hoogenboom2019dssat}
Hoogenboom, G., Porter, C., Boote, K., Shelia, V., Wilkens, P., Singh, U.,
  White, J., Asseng, S., Lizaso, J., Moreno, L., et~al. (2019).
\newblock The dssat crop modeling ecosystem.
\newblock {\em Advances in crop modelling for a sustainable agriculture}, pages
  173--216.

\bibitem[Huet et~al., 2022]{huet2022coping}
Huet, E., Adam, M., Traore, B., Giller, K., and Descheemaeker, K. (2022).
\newblock Coping with cereal production risks due to the vagaries of weather,
  labour shortages and input markets through management in southern mali.
\newblock {\em European Journal of Agronomy}, 140:126587.

\bibitem[Jourdain et~al., 2020]{jourdain2020farmers}
Jourdain, D., Lairez, J., Striffler, B., and Affholder, F. (2020).
\newblock Farmers’ preference for cropping systems and the development of
  sustainable intensification: a choice experiment approach.
\newblock {\em Review of Agricultural, Food and Environmental Studies},
  101(4):417--437.

\bibitem[Kalaji et~al., 2017]{kalaji2017comparison}
Kalaji, H.~M., Dabrowski, P., Cetner, M.~D., Samborska, I.~A., Lukasik, I.,
  Brestic, M., Zivcak, M., Tomasz, H., Mojski, J., Kociel, H., et~al. (2017).
\newblock A comparison between different chlorophyll content meters under
  nutrient deficiency conditions.
\newblock {\em Journal of Plant Nutrition}, 40(7):1024--1034.

\bibitem[Laird and Ware, 1982]{laird1982random}
Laird, N.~M. and Ware, J.~H. (1982).
\newblock Random-effects models for longitudinal data.
\newblock {\em Biometrics}, pages 963--974.

\bibitem[Lattimore and Szepesv{\'a}ri, 2020]{lattimore2020bandit}
Lattimore, T. and Szepesv{\'a}ri, C. (2020).
\newblock {\em Bandit algorithms}.
\newblock Cambridge University Press.

\bibitem[Mandelbrot, 1997]{mandelbrot1997variation}
Mandelbrot, B.~B. (1997).
\newblock The variation of certain speculative prices.
\newblock In {\em Fractals and scaling in finance}, pages 371--418. Springer.

\bibitem[Massart, 1990]{massart1990}
Massart, P. (1990).
\newblock The tight constant in the dvoretzky-kiefer-wolfowitz inequality.
\newblock {\em Annals of Probability}, 18.

\bibitem[McCown, 2002]{mccown2002changing}
McCown, R.~L. (2002).
\newblock Changing systems for supporting farmers' decisions: problems,
  paradigms, and prospects.
\newblock {\em Agricultural systems}, 74(1):179--220.

\bibitem[Meisinger and Delgado, 2002]{meisinger2002principles}
Meisinger, J.~J. and Delgado, J.~A. (2002).
\newblock Principles for managing nitrogen leaching.
\newblock {\em Journal of soil and water conservation}, 57(6):485--498.

\bibitem[Menapace et~al., 2013]{menapace2013risk}
Menapace, L., Colson, G., and Raffaelli, R. (2013).
\newblock Risk aversion, subjective beliefs, and farmer risk management
  strategies.
\newblock {\em American Journal of Agricultural Economics}, 95(2):384--389.

\bibitem[Morris et~al., 2018]{morris2018strengths}
Morris, T.~F., Murrell, T.~S., Beegle, D.~B., Camberato, J.~J., Ferguson,
  R.~B., Grove, J., Ketterings, Q., Kyveryga, P.~M., Laboski, C.~A., McGrath,
  J.~M., et~al. (2018).
\newblock Strengths and limitations of nitrogen rate recommendations for corn
  and opportunities for improvement.
\newblock {\em Agronomy Journal}, 110(1):1.

\bibitem[Naudin et~al., 2010]{naudin2010impact}
Naudin, K., Goz{\'e}, E., Balarabe, O., Giller, K.~E., and Scopel, E. (2010).
\newblock Impact of no tillage and mulching practices on cotton production in
  north cameroon: a multi-locational on-farm assessment.
\newblock {\em Soil and Tillage Research}, 108(1-2):68--76.

\bibitem[Perchet et~al., 2015]{perchet2015batch}
Perchet, V., Rigollet, P., Chassang, S., and Snowberg, E. (2015).
\newblock Batched bandit problems.
\newblock In Gr{\"{u}}nwald, P., Hazan, E., and Kale, S., editors, {\em
  Proceedings of The 28th Conference on Learning Theory, {COLT} 2015, Paris,
  France, July 3-6, 2015}, volume~40 of {\em {JMLR} Workshop and Conference
  Proceedings}, page 1456. JMLR.org.

\bibitem[Richardson and Wright, 1984]{richardson1984wgen}
Richardson, C.~W. and Wright, D.~A. (1984).
\newblock Wgen: A model for generating daily weather variables.
\newblock {\em ARS (USA)}.

\bibitem[Ripoche et~al., 2015]{ripoche2015cotton}
Ripoche, A., Cr{\'e}tenet, M., Corbeels, M., Affholder, F., Naudin, K.,
  Sissoko, F., Douzet, J.-M., and Tittonell, P. (2015).
\newblock Cotton as an entry point for soil fertility maintenance and food crop
  productivity in savannah agroecosystems--evidence from a long-term experiment
  in southern mali.
\newblock {\em Field crops research}, 177:37--48.

\bibitem[Robbins, 1952]{robbins1952some}
Robbins, H. (1952).
\newblock Some aspects of the sequential design of experiments.
\newblock {\em Bulletin of the American Mathematical Society}, 58(5):527--535.

\bibitem[Soltani and Hoogenboom, 2003]{soltani2003statistical}
Soltani, A. and Hoogenboom, G. (2003).
\newblock A statistical comparison of the stochastic weather generators wgen
  and simmeteo.
\newblock {\em Climate Research}, 24(3):215--230.

\bibitem[Tack and Holt, 2016]{tack2016influence}
Tack, J.~B. and Holt, M.~T. (2016).
\newblock The influence of weather extremes on the spatial correlation of corn
  yields.
\newblock {\em Climatic Change}, 134(1-2):299--309.

\bibitem[Tamkin et~al., 2020]{tamkindistributionally}
Tamkin, A., Keramati, R., Dann, C., and Brunskill, E. (2020).
\newblock Distributionally-aware exploration for cvar bandits.
\newblock In {\em NeurIPS 2019 Workshop on Safety and Robustness in Decision
  Making; RLDM 2019}.

\bibitem[Ten~Berge et~al., 2019]{ten2019maize}
Ten~Berge, H.~F., Hijbeek, R., Van~Loon, M., Rurinda, J., Tesfaye, K., Zingore,
  S., Craufurd, P., van Heerwaarden, J., Brentrup, F., Schr{\"o}der, J.~J.,
  et~al. (2019).
\newblock Maize crop nutrient input requirements for food security in
  sub-saharan africa.
\newblock {\em Global Food Security}, 23:9--21.

\bibitem[Thomas and Learned-Miller, 2019]{thomas2019concentration}
Thomas, P. and Learned-Miller, E. (2019).
\newblock Concentration inequalities for conditional value at risk.
\newblock In {\em International Conference on Machine Learning}, pages
  6225--6233. PMLR.

\bibitem[Thompson, 1933]{thompson1933likelihood}
Thompson, W.~R. (1933).
\newblock On the likelihood that one unknown probability exceeds another in
  view of the evidence of two samples.
\newblock {\em Biometrika}, 25(3/4):285--294.

\bibitem[Tilman et~al., 2002]{tilman2002agricultural}
Tilman, D., Cassman, K.~G., Matson, P.~A., Naylor, R., and Polasky, S. (2002).
\newblock Agricultural sustainability and intensive production practices.
\newblock {\em Nature}, 418(6898):671--677.

\bibitem[Traore et~al., 2017]{traore2017modelling}
Traore, B., Descheemaeker, K., Van~Wijk, M.~T., Corbeels, M., Supit, I., and
  Giller, K.~E. (2017).
\newblock Modelling cereal crops to assess future climate risk for family food
  self-sufficiency in southern mali.
\newblock {\em Field Crops Research}, 201:133--145.

\bibitem[Vanlauwe et~al., 2011]{vanlauwe2011agronomic}
Vanlauwe, B., Kihara, J., Chivenge, P., Pypers, P., Coe, R., and Six, J.
  (2011).
\newblock Agronomic use efficiency of n fertilizer in maize-based systems in
  sub-saharan africa within the context of integrated soil fertility
  management.
\newblock {\em Plant and soil}, 339(1):35--50.

\end{thebibliography}
